\RequirePackage[l2tabu,orthodox]{nag}
\documentclass
[letterpaper,11pt,]
{article}

\usepackage[utf8]{inputenc}
\usepackage{etex}
\usepackage{verbatim}
\usepackage{xspace,enumerate}
\usepackage[dvipsnames]{xcolor}
\usepackage[T1]{fontenc}
\usepackage[full]{textcomp}
\usepackage[american]{babel}
\usepackage{mathtools}
\usepackage{amsthm}
\usepackage[
letterpaper,
top=1in,
bottom=1in,
left=1in,
right=1in]{geometry}
\usepackage{newpxtext} % T1, lining figures in math, osf in text
\usepackage{textcomp} % required for special glyphs
\usepackage[varg,bigdelims]{newpxmath}
\usepackage[scr=rsfso]{mathalfa}% \mathscr is fancier than \mathcal
\usepackage{bm} % load after all math to give access to bold math
\let\mathbb\varmathbb
\usepackage{microtype}
\usepackage[pagebackref,colorlinks=true,urlcolor=blue,linkcolor=blue,citecolor=OliveGreen]{hyperref}
\usepackage[capitalise,nameinlink]{cleveref}
\crefname{lemma}{Lemma}{Lemmas}
\crefname{fact}{Fact}{Facts}
\crefname{theorem}{Theorem}{Theorems}
\crefname{corollary}{Corollary}{Corollaries}
\crefname{claim}{Claim}{Claims}
\crefname{example}{Example}{Examples}
\crefname{algorithm}{Algorithm}{Algorithms}
\crefname{problem}{Problem}{Problems}
\crefname{definition}{Definition}{Definitions}
\crefname{model}{Model}{Models}
\crefname{exercise}{Exercise}{Exercises}
\crefname{condition}{Condition}{Conditions}
\usepackage{amsthm}

\newtheorem{theorem}{Theorem}[section]
\newtheorem*{theorem*}{Theorem}
\newtheorem{lemma}[theorem]{Lemma}
\newtheorem*{lemma*}{Lemma}
\newtheorem{fact}[theorem]{Fact}
\newtheorem*{fact*}{Fact}

\newtheorem*{proposition*}{Proposition}
\newtheorem{corollary}[theorem]{Corollary}
\newtheorem*{corollary*}{Corollary}

\newtheorem*{hypothesis*}{Hypothesis}

\newtheorem*{conjecture*}{Conjecture}
\theoremstyle{definition}

\newtheorem*{definition*}{Definition}

\newtheorem*{model*}{Model}

\newtheorem*{construction*}{Construction}

\newtheorem*{example*}{Example}

\newtheorem*{question*}{Question}
\newtheorem{algorithm}[theorem]{Algorithm}
\newtheorem*{algorithm*}{Algorithm}

\newtheorem*{assumption*}{Assumption}

\newtheorem*{problem*}{Problem}

\newtheorem*{openquestion*}{Open Question}
\theoremstyle{remark}

\newtheorem*{claim*}{Claim}

\newtheorem*{remark*}{Remark}

\newtheorem*{observation*}{Observation}
\usepackage{paralist}
\let\originalleft\left
\let\originalright\right
\renewcommand{\left}{\mathopen{}\mathclose\bgroup\originalleft}
\renewcommand{\right}{\aftergroup\egroup\originalright}
\usepackage{turnstile}
\usepackage{mdframed}
\usepackage{tikz}
\usetikzlibrary{positioning}
\usepackage{caption}
\DeclareCaptionType{Algorithm}
\usepackage{newfloat}
\usepackage{array}
\usepackage{subfig}
\usepackage{bbm}
\usepackage{xparse}
\usepackage{amsthm} % for \@addpunct
\makeatletter
\let\latexparagraph\paragraph
\RenewDocumentCommand{\paragraph}{som}{%
  \IfBooleanTF{#1}
    {\latexparagraph*{#3}}
    {\IfNoValueTF{#2}
       {\latexparagraph{\maybe@addperiod{#3}}}
       {\latexparagraph[#2]{\maybe@addperiod{#3}}}%
  }%
}
\newcommand{\maybe@addperiod}[1]{%
  #1\@addpunct{.}%
}
\makeatother

\newcommand{\Authornote}[2]{{\sffamily\small\color{red}{[#1: #2]}}}

\newcommand{\Tnote}{\Authornote{T}}

\usepackage{boxedminipage}
\newcommand{\paren}[1]{(#1)}
\newcommand{\Paren}[1]{\left(#1\right)}

\newcommand{\brac}[1]{[#1]}
\newcommand{\Brac}[1]{\left[#1\right]}

\newcommand{\abs}[1]{\lvert#1\rvert}
\newcommand{\Abs}[1]{\left\lvert#1\right\rvert}

\newcommand{\card}[1]{\lvert#1\rvert}
\newcommand{\Card}[1]{\left\lvert#1\right\rvert}

\newcommand{\set}[1]{\{#1\}}
\newcommand{\Set}[1]{\left\{#1\right\}}

\newcommand{\norm}[1]{\lVert#1\rVert}
\newcommand{\Norm}[1]{\left\lVert#1\right\rVert}

\newcommand{\Normt}[1]{\Norm{#1}_2}

\newcommand{\Snorm}[1]{\Norm{#1}^2}

\newcommand{\normo}[1]{\norm{#1}_1}
\newcommand{\Normo}[1]{\Norm{#1}_1}

\newcommand{\normi}[1]{\norm{#1}_\infty}
\newcommand{\Normi}[1]{\Norm{#1}_\infty}

\newcommand{\iprod}[1]{\langle#1\rangle}
\newcommand{\Iprod}[1]{\left\langle#1\right\rangle}

\newcommand{\Esymb}{\mathbb{E}}
\newcommand{\Psymb}{\mathbb{P}}

\DeclareMathOperator*{\E}{\Esymb}

\DeclareMathOperator*{\ProbOp}{\Psymb}
\renewcommand{\Pr}{\ProbOp}
\newcommand{\given}{\mathrel{}\middle\vert\mathrel{}}
\newcommand{\suchthat}{\;\middle\vert\;}
\newcommand{\sge}{\succeq}

\renewcommand{\ij}{{ij}}
\newcommand\bdot\bullet
\DeclareMathOperator{\Tr}{Tr}

\DeclareMathOperator{\argmax}{argmax}
\DeclareMathOperator{\argmin}{argmin}
\DeclareMathOperator{\polylog}{polylog}
\DeclareMathOperator{\supp}{supp}

\DeclareMathOperator{\sign}{sign}

\newcommand{\N}{\mathbb N}
\newcommand{\R}{\mathbb R}

\newcommand{\cA}{\mathcal A}

\newcommand{\cC}{\mathcal C}

\newcommand{\cG}{\mathcal G}

\newcommand{\cK}{\mathcal K}
\newcommand{\cL}{\mathcal L}

\newcommand{\cP}{\mathcal P}

\newcommand{\cS}{\mathcal S}

\newcommand{\cU}{\mathcal U}

\newcommand{\bbP}{\mathbb P}

\renewcommand{\leq}{\leqslant}
\renewcommand{\le}{\leqslant}
\renewcommand{\geq}{\geqslant}
\renewcommand{\ge}{\geqslant}
\let\epsilon=\varepsilon
\numberwithin{equation}{section}
\newcommand\MYcurrentlabel{xxx}
\newcommand{\MYstore}[2]{%
  \global\expandafter \def \csname MYMEMORY #1 \endcsname{#2}%
}
\newcommand{\MYload}[1]{%
  \csname MYMEMORY #1 \endcsname%
}
\newcommand{\MYnewlabel}[1]{%
  \renewcommand\MYcurrentlabel{#1}%
  \MYoldlabel{#1}%
}
\newcommand{\MYdummylabel}[1]{}
\newcommand{\torestate}[1]{%
  \let\MYoldlabel\label%
  \let\label\MYnewlabel%
  #1%
  \MYstore{\MYcurrentlabel}{#1}%
  \let\label\MYoldlabel%
}
\newcommand{\restatetheorem}[1]{%
  \let\MYoldlabel\label
  \let\label\MYdummylabel
  \begin{theorem*}[Restatement of \cref{#1}]
    \MYload{#1}
  \end{theorem*}
  \let\label\MYoldlabel
}
\newcommand{\restatelemma}[1]{%
  \let\MYoldlabel\label
  \let\label\MYdummylabel
  \begin{lemma*}[Restatement of \cref{#1}]
    \MYload{#1}
  \end{lemma*}
  \let\label\MYoldlabel
}
\newcommand{\restateprop}[1]{%
  \let\MYoldlabel\label
  \let\label\MYdummylabel
  \begin{proposition*}[Restatement of \cref{#1}]
    \MYload{#1}
  \end{proposition*}
  \let\label\MYoldlabel
}
\newcommand{\restatefact}[1]{%
  \let\MYoldlabel\label
  \let\label\MYdummylabel
  \begin{fact*}[Restatement of \cref{#1}]
    \MYload{#1}
  \end{fact*}
  \let\label\MYoldlabel
}
\newcommand{\restate}[1]{%
  \let\MYoldlabel\label
  \let\label\MYdummylabel
  \MYload{#1}
  \let\label\MYoldlabel
}
\newcommand{\eps}{\epsilon}

\allowdisplaybreaks
\newcommand*{\Id}{\mathrm{Id}}

\newcommand*{\Normf}[1]{\Norm{#1}_{\mathrm{F}}}
\newcommand{\ind}[1]{\mathbf{1}_{\Brac{#1}}}
\providecommand{\todo}{{\color{red}{\textbf{TODO }}}}

\newcommand*{\transpose}[1]{{#1}{}^{\mkern-1.5mu\mathsf{T}}}

\providecommand{\pdset}{\tilde{\Omega}}%_{\overline{\Omega}_{b}, p,  \cD_t}}

\title{
	Sparse PCA Beyond Covariance Thresholding
	\thanks{This project has received funding from the European Research Council (ERC) under the European Union's Horizon 2020 research and innovation programme (grant agreement No 815464).}
}

\author{
	Gleb Novikov\thanks{ETH Z\"urich.}
}

\date{}

\begin{document}
	
	\pagestyle{empty}
	
	% MAKE TITLE

	\maketitle
	\thispagestyle{empty} % seems to be required here to avoid page number on first page

	% ABSTRACT
	
	\begin{abstract}
		In the Wishart model for sparse PCA
we are given $n$ samples $\bm Y_1,\ldots,  \bm Y_n$ drawn independently 
from a $d$-dimensional Gaussian distribution $N({0, \Id + \beta vv^\top})$, where $\beta > 0$ and $v\in \R^d$ is a 
$k$-sparse unit vector, and we wish to recover $v$ (up to sign).

We show that if $n \ge \Omega(d)$, 
then for every $t \ll k$ there exists an algorithm running in time $n\cdot d^{O(t)}$ that solves this problem as long as
\[
\beta \gtrsim \frac{k}{\sqrt{nt}}\sqrt{\ln({2 + td/k^2})}\,.
\]
Prior to this work, the best polynomial time algorithm in the regime $k\approx \sqrt{d}$, 
called \emph{Covariance Thresholding} 
(proposed in \cite{KNV15} and analyzed in \cite{DBLP:conf/nips/DeshpandeM14}), 
required $\beta \gtrsim \frac{k}{\sqrt{n}}\sqrt{\ln({2 + d/k^2})}$.
For large enough constant $t$ our algorithm runs in polynomial time and has 
better guarantees than Covariance Thresholding.
Previously known algorithms with such guarantees required quasi-polynomial time $d^{O(\log d)}$.

Our idea is based on the idea of \cite{AKS98} for reducing the clique size in the planted clique problem.
Moreover, we show that it is possible to combine our techniques with 
recent results on sparse PCA with symmetric heavy-tailed noise \cite{dNNS22}. 
Their model generalizes both sparse PCA and the planted clique problem.
In particular, 
in the regime $k \approx \sqrt{d}$ we get the first polynomial time algorithm that works with symmetric heavy-tailed noise, 
while the algorithm from \cite{dNNS22} requires quasi-polynomial time in these settings. As a consequence, we get an algorithm that solves a problem that captures both sparse PCA and planted clique and achieves best known guarantees for both of them.

In addition, we show that our techniques work with sparse PCA with adversarial perturbations studied in \cite{dKNS20}. 
This model generalizes not only sparse PCA, 
but also the sparse planted vector problem.
As a consequence, we provide polynomial time algorithms for 
the sparse planted vector problem that have better guarantees than
the state of the art in some regimes.
	\end{abstract}
	
	\clearpage
	
	% TOC

	% assumes microtype
	\microtypesetup{protrusion=false}
	\tableofcontents{}
	\microtypesetup{protrusion=true}

	\clearpage
	
	\pagestyle{plain}
	\setcounter{page}{1}
	
	% SECTION

	\section{Introduction} \label{sec:introduction}

We study sparse principal component analysis in the \emph{Wishart} and \emph{Wigner} models. 
First we describe the Wishart model (that is sometimes also called the
	\emph{spiked covariance model}).
In this model, we are given $n$ samples
$\bm Y_1,\ldots,  \bm Y_n$ drawn\footnote{We use boldface to denote random variables.}  independently 
from $d$-dimensional Gaussian distribution
$N\Paren{0, \Id + \beta vv^\top}$, 
where $\beta > 0$ and $v\in \R^d$ is a 
$k$-sparse\footnote{That is, this vector has at most $k$ non-zero coordinates.} unit vector. 
The goal is to compute an estimator $\hat{\bm v}$ such that $\Norm{\hat{\bm v}} = 1$ 
and
$\Abs{\Iprod{\hat{\bm v}, v}}$ is close to $1$ (say, is greater than $0.99$) 
with high probability\footnote{It is impossible to recover the sign of $v$ from $\bm Y_1,\ldots,  \bm Y_n$.}. 
In this paper we mostly focus on the regime when the number of samples $n$ is greater than the dimension $d$, 
and in this section of the paper we always assume that\footnote{ We hide absolute constant multiplicative factors using the standard notations $O(\cdot), \Omega(\cdot), \lesssim, \gtrsim$.}  $n\ge \Omega(d)$ (unless stated otherwise).

\paragraph{Classical settings.}
The standard approach in covariance estimation is to consider the empirical covariance $\tfrac{1}{n}\bm Y^\top \bm Y$ 
(where $\bm Y$ is the matrix with rows $\bm Y_1,\ldots,  \bm Y_n$).
The top eigenvector of $\tfrac{1}{n}\bm Y^\top \bm Y$
is highly correlated with $v$ or $-v$ as long as $\beta \gtrsim \sqrt{\frac{d}{n}}$, 
and in non-sparse settings ($k = d$) these guarantees are information theoretically optimal.
For $k < d$, there exists an estimator with better guarantees. 
It uses exhaustive search over all $\binom{d}{k}$ candidates for the support of $v$ and is close to $v$ or $-v$ 
iff $\beta \gtrsim \sqrt{\frac{k\log\Paren{de/k}}{n}}$, 
and these guarantees are information theoretically optimal in sparse settings \cite{amini2009,DBLP:journals/corr/abs-1304-0828,berthet2013}. 

As was observed in \cite{johnstone-lu-dt}, known algorithmic guarantees for sparse PCA are strictly worse than the statistical
 guarantees described above. 
In the regime $k \gg \sqrt{d}$, no polynomial time algorithm is known to work if $\beta \lesssim \sqrt{\frac{d}{n}}$ 
(recall that if $\beta \gtrsim \sqrt{\frac{d}{n}}$, the top eigenvector of the empirical covariance is a good estimator). 
\cite{johnstone-lu-dt} proposed a polynomial time algorithm 
(called \emph{Diagonal Thresholding}) that finds an estimator that is close to $v$ or $-v$ as long as 
$\beta \gtrsim k\sqrt{\frac{\log d}{n}}$, which is better than the top eigenvector of $\bm Y^\top \bm Y$ if $k \ll \sqrt{d}$, 
but is worse than the information-theoretically optimal estimator by a factor $\sqrt{k}$. 

Later many computational lower bounds of different kind appeared: reductions from the planted clique problem \cite{DBLP:conf/colt/BerthetR13, DBLP:journals/corr/abs-1304-0828, WBS16, GMZ17, BBH18, DBLP:conf/colt/BrennanB19}, 
low degree polynomial lower bounds \cite{DKWB19, dKNS20}, 
statistical query lower bounds \cite{BBH+21}, SDP and sum-of-squares lower bounds \cite{KNV15, DBLP:conf/nips/MaW15, PR22}, 
lower bounds for Markov chain Monte Carlo methods \cite{AWZ20}.
These lower bounds suggest that the algorithms described above
should have optimal guarantees in the regimes $k \ll \sqrt{d}$ (Diagonal Thresholding) and $k\gg \sqrt{d}$ (the top eigenvector),
so it is unlikely that there exist efficient algorithms with significantly better guarantees if $k \ll \sqrt{d}$ or $k \gg \sqrt{d}$.

The regime $k \approx \sqrt{d}$ is more interesting. 
For a long time no efficiently computable estimator with provable
 guarantees better than the top eigenvector of $\bm Y^\top \bm Y$ or than Diagonal Thresholding was known, until
\cite{DBLP:conf/nips/DeshpandeM14} proved that a polynomial time algorithm (called \emph{Covariance Thresholding}) computes an estimator that is close
to $v$ or $-v$ as long as $\beta\gtrsim k\sqrt{\frac{\log\Paren{2+ d/k^2}}{n}}$.
This estimator can exploit sparsity if $k < \sqrt{d}$ and 
is better than Diagonal Thresholding and the top eigenvector of the empirical covariance  in the regime $d^{1/2 - o(1)} < k< \sqrt{d}$. 
%The analysis of Covariance Thresholding is very delicate and is much more complicated than that of Diagonal Thresholding and the top eigenvector.

These results show that in order to work with smaller signal strength $\beta$, 
one needs either to work with larger number of samples $n$, or to work with a sparser vector $v$ (i.e. smaller $k$).
\cite{DKWB19} (and independently \cite{HSV20}) showed that in some regimes there is another option: 
one can (smoothly) increase the running time needed to compute the estimator in order to work with smaller signal strength. 
Concretely, they showed that for $1 \le t \le k / \log d$ 
there exists an estimator that can be computed in time $d^{O(t)}$ 
(via \emph{limited brute force}) and is close to $v$ or $-v$ as long as
$\beta \gtrsim k\sqrt{\frac{\log d}{tn}}$. 
The following example illustrates their result: 
For some  $n, d, k\in \N$, let $\beta_{\text{DT}}$ be the smallest signal strength such that Diagonal Thresholding,
given an instance $\bm Y$ of sparse PCA with $n$ samples, dimension $d$, sparsity $k$ and signal strength $\beta_{\text{DT}}$,
finds a unit vector $\hat{\bm v}_{\text{DT}}$ such that $\Abs{\Iprod{\hat{\bm v}_{\text{DT}}, v}} \ge 0.99$ with high probability. 
Now suppose that for the same $n, d, k$, we are given an instance $\bm Y'$ of sparse PCA with smaller signal strength 
$\beta_{\text{new}} = 0.01 \cdot \beta_{\text{DT}}$. 
Then their result implies that there 
exists a polynomial time algorithm that, given $\bm Y'$, finds a unit vector $\hat{\bm v}_{\text{new}}$ such that 
$\Abs{\Iprod{\hat{\bm v}_{\text{new}}, v}} \ge 0.99$ with high probability. 

However, the approach of \cite{DKWB19} and \cite{HSV20} is not compatible with the optimal guarantees in the regime $k \approx \sqrt{d}$. 
More precisely, if we define $\beta_{\text{CT}}$ as the smallest signal strength for Covariance Thresholding 
(in the same way as we defined $\beta_{\text{DT}}$ for Diagonal Thresholding), 
then the limited brute force that works with signal strength  $0.01 \cdot \beta_{\text{CT}}$ requires $t$ to be at least $\log d$ and hence
runs in quasi-polynomial time $d^{O\Paren{\log d}}$. Prior to this work it was the fastest algorithm in this regime.

Our result shows that it is possible to smoothly increasy running time in order to work with smaller signal strength as long as 
$k \le O(\sqrt{d})$ and $n\ge \Omega\Paren{d}$.
It can be informally described as follows: Let $\cA$ be an arbitrary currently known polynomial time algorithm for sparse PCA. 
Let $\beta_{\cA}$ be the smallest signal strength such that $\cA$,
given an instance $\bm Y$ of sparse PCA with, dimension $d$, $n \ge \Omega(d)$ samples, sparsity $k\le O\Paren{\sqrt{d}}$,
and signal strength $\beta_{\cA}$,
finds a unit vector $\hat{\bm v}_{\cA}$ such that $\Abs{\Iprod{\hat{\bm v}_{\cA}, v}} \ge 0.99$ with high probability. 
For arbitrary constant $C \ge 1$, let $\beta_C = \frac{1}{C}\beta_{\cA}$. 
Then there exists a polynomial time\footnote{The degree of the polynomial depends on $C$.} algorithm, that, 
given an instance of sparse PCA $\bm Y'$ with signal strength $\beta_C$ and the same parameters $n, d, k$ as for $\bm Y$,  
finds a unit vector $\hat{\bm v}_{\text{new}}$ such that 
$\Abs{\Iprod{\hat{\bm v}_{\text{new}}, v}} \ge 0.99$ with high probability.

In particular, our result implies that there exists a polynomial time algorithm that works with signal strength  $0.01 \cdot \beta_{\text{CT}}$, 
which is a significant improvement compared to the best previously known (quasi-polynomial time) algorithm. 
Moreover, our result also implies the first polynomial time algorithm that can exploit sparsity and has better guarantees than the top eigenvector
 even in the regime $k \ge \sqrt{d}$ (as long as $k \le O\Paren{\sqrt{d}}$).

\paragraph{Semidefinite programming and adversarial perturbations} $\,$ 
\cite{DBLP:conf/nips/dAspremontGJL04} introduced \emph{basic SDP} for sparse PCA. 
basic SDP achieves guarantees of 
both the top eigenvector of the empirical covariance and Diagonal Thresholding. 
Later \cite{dKNS20} proved that it also achieves the guarantees of Covariance Thresholding, 
and hence captures the best currently known polynomial time guarantees. 

Moreover, \cite{dKNS20} showed that basic SDP also works with adversarial perturbations.
 More precisely, if a small (adversarially chosen) value $E_{ij}$
 is added to every entry $\bm Y_{ij}$ of an instance of sparse PCA, basic SDP still recovers $v$ or $-v$ with high probability.  
Known estimators that are not based on semidefinite programming, 
 including top eigenvector of the empirical covariance, Diagonal Thresholding, 
 Covariance Thresholding and limited brute force, do not work with adversarial perturbations. 
\cite{dKNS20} also provided a family of algorithms based on sum-of-squares relaxations that
 work with adversarial perturbations and
 achieves the guarantees of limited brute force from \cite{DKWB19}.
 
 Similar to non-adversarial case, basic SDP and limited brute force based on sum-of-squares are not compatible with each other:
 in the regime $k\approx \sqrt{d}$, 
 the sum-of-squares approach from \cite{dKNS20} requires degree $\log d$ 
  in order to achieve better guarantees than basic SDP,
 so the corresponding estimator can be computed only in quasi-polynomial time.
 
 We show that our technique also works with adversarial perturbations. 
 We remark that we do not use higher degree sum-of-squares, 
 but only basic SDP for sparse PCA (with some preprocessing and postprocessing steps).
 
One of the applications of our result is an improvement in the planted sparse vector problem. For this problem we focus on the regime $\Omega(d) < n < d$.
In this problem, 
we are given an $n$-dimensional
 subspace of $\R^d$ that contains a sparse vector, 
and the goal is to estimate this vector. This problem was extensively studied in literature in different settings \cite{HD13, DBLP:conf/stoc/BarakKS14, DBLP:conf/stoc/HopkinsSSS16, DBLP:journals/corr/abs-2001-06970, DBLP:journals/corr/abs-2105-15081, DBLP:conf/colt/ZadikSWB22, DBLP:conf/colt/DiakonikolasK22a}.
It is not hard to see\footnote{See the discussion before \cref{cor:planted-vector}.} 
that this problem in the \emph{Gaussian basis} model (in the sense of \cite{DBLP:journals/corr/abs-2105-15081}) 
is a special case of sparse PCA with small perturbations. 
Our result shows that as long as $k\le \sqrt{td}$, 
there exists a $d^{O(t)}$ time algorithm for this problem. 
Previously known polynomial time algorithms in the regime 
$n\ge \Omega(d)$ required $k \le C \sqrt{d}$ for some absolute constant $C$ and did not work for $k > C \sqrt{d}$.
Lower bounds against restricted computational models \cite{dKNS20, DBLP:journals/tit/DingKWB21, DBLP:journals/corr/abs-2301-11124}
suggest that in the regime $n\ge \Omega(d)$
this problem is unlikely to be solvable in polynomial time if $k \gg \sqrt{d}$.

\paragraph{The Wigner model and symmetric noise} Our results can be naturally applied also to the {Wigner} model. In this model, 
we are given $\bm Y = \lambda vv^\top + \bm W$, 
where $\lambda > 0$, $v\in \R^d$ is a $k$-sparse unit vector, and $\bm W\sim N\Paren{0,1}^{d\times d}$.
For this model, Covariance Thresholding finds an estimator highly correlated with with $v$ or $-v$ as long as 
$\lambda \gtrsim k\sqrt{\log\Paren{2+ d/k^2}}$, 
while limited brute force from \cite{DKWB19} 
computes in time $d^{O(t)}$ an estimator close to $v$ or $-v$ as long as $\lambda \gtrsim k\sqrt{\frac{\log d}{t}}$.
As in the Wishart model, these approaches are not compatible in the regime $k \approx \sqrt{d}$. 
Our techniques can be naturally applied to the Wigner model, 
leading to the best known algorithms for this problem. 

As in the Wishart model, our techniques also work with adversarial perturbations. 
Moreover, our approach is compatible with the recent study of Sparse PCA with symmetric noise \cite{dNNS22}. 
In this model, Gaussian noise $\bm W$ is replaced by an arbitrary  noise $\bm N$ with symmetric about zero
 independent entries that are only guaranteed to be bounded  by $1$ with probability\footnote{Note that even the first moment is not required to exist.}  $\Omega\Paren{1}$. 
They proposed a quasi-polynomial algorithm for sparse PCA in these settings and provided evidence that in the regime $k\ll \sqrt{d}$
 this running time cannot be improved (via reduction from the planted clique problem). 
Combining their algorithm with our approach, 
we show that in the regime $k\approx \sqrt{d}$ there exists a polynomial time algorithm that solves this problem.

%
%\Gnote{Plan:}
%\begin{enumerate}
%	\item Brief history of sparse pca (Wishart model), DT, basic SDP.
%	\item Covariance Thresholding, clear formulation with quantifiers over constants.
%	\item limited brute force, running time vs signal strength. Explanation why previous results didn't work in the regime where CT is the best algorithm (i.e. that spectral norm bound for principal submatrices is meaningless if they have size less than $\log d$).
%	\item Announcement of our results, comparison with CT.
%	\item Reference to the lower bound \Gnote{perhaps we can formulate a conjecture that our family of algorithms is optimal?}
%	\item Adversarial perturbations, comparison with old spca paper.
%	\item Wigner model: classical settings, adversarial perturbations.
%	\item Wigner model with symmetric noise, comparison with soda paper and planted clique.
%\end{enumerate}

\subsection{Results}
Before stating the results, observe that one can write an instance $\bm Y$ sparse PCA in the Wishart model as $\bm Y  = \sqrt{\beta}\bm uv^\top + \bm W$, where $u\sim N(0,1)^n$ and $\bm W \sim N(0,1)^{n\times d}$ are independent.

\paragraph{Classical settings} 
Our first result is estimating $v$ in the Wishart model in classical settings (without perturbations).
\begin{theorem}[The Wishart model]\label{thm:wishart-classical-results}
	Let $n,d,k,t\in \N$, $\beta > 0$. 
	Let 
	$
	\bm Y = \sqrt{\beta}\bm u v^\top + \bm W\,,
	$
	where $\bm u \sim N(0,1)^n$, $v\in \R^d$ is a $k$-sparse unit vector, $\bm W\sim N(0,1)^{n\times d}$ independent of $\bm u$.
	
	There exists an absolute constant $C > 1$, such that if $n \ge Ck$, 
	$k\ge C t\log^2 d$ and
	\[
	\beta \ge C \frac{k}{\sqrt{tn}}\sqrt{\log\Paren{2 + \frac{td}{k^2}\Paren{1 + \frac{d}{n}}  }}\,,
	\]
	then there exists an algorithm that, given $\bm Y$, $k$ and $t$, 
	in time $n \cdot d^{O(t)}$ outputs a unit vector $\hat{\bm v}$ such that
	with probability $1-o(1)$ as $d\to \infty$,
	\[
	\abs{\iprod{\hat{\bm v}, v}} \ge 0.99\,.
	\]
\end{theorem}
Let us compare our guarantees with previously known estimators. For simplicity we assume $n \ge \Omega(d)$.
For this regime, best estimators known prior to this work and their guarantees are listed in Table \ref{table:classical}.

\begin{table}[h!]
\begin{center}
	\begin{tabular}{| c | c | c |}
		\hline
		Estimator & Signal Strength & Time Complexity\\
		\hline\hline
		Statistically optimal estimator & $ \beta\gtrsim \sqrt{\frac{k}{n}\log \Paren{ed/k}}$ & $n\cdot d^{O(k)}$ \\[5pt]
		\hline
		Top eigenvector of the empirical covariance & $\beta\gtrsim \sqrt{\frac{d}{n}}$ & $n\cdot d^{O(1)}$ \\ 
		\hline
		Covariance Thresholding & $\beta\gtrsim \frac{k}{\sqrt{n}}\sqrt{\log \Paren{2 + d/k^2}}$ & $n\cdot d^{O(1)}$ \\ 
		\hline
		Limited brute force from \cite{DKWB19} & $\beta\gtrsim \frac{k}{\sqrt{tn}}\sqrt{\log d}$ & $n\cdot d^{O(t)}$ \\ 
\hline
		Our estimator & $\beta\gtrsim \frac{k}{\sqrt{tn}}\sqrt{\log \Paren{2 + td/k^2}}$ & $n\cdot d^{O(t)}$ \\ 
\hline
	\end{tabular}
\end{center}
 \caption{Estimators for sparse PCA in the Wishart model (assuming $n\ge \Omega(d)$ and $t \le k / \polylog(d)$).}\label{table:classical}
\end{table}

For $k \le d^{1/2-\Omega(1)}$ (say, $k \le d^{0.49}$), the guarantees of the algorithm from \cite{DKWB19} are similar to ours (up to a constant factor). 
For $k \ge d^{1/2-o(1)}$ our algorithm can work with asymptotically smaller signal strength (with the same running time).

To compare with Covariance Thresholding and the top eigenvector of the empirical covariance, consider the regime
$k = \Theta\Paren{\sqrt{d}}$. 
Note in this regime both Covariance Thresholding and the top eigenvector require
\[
\beta \ge c \sqrt{d/n}
\]
for some specific constant $c$ (that depends on $\sqrt{d}/k$), and they do not work for smaller $\beta$.
Our condition on $\beta$ in these settings is
\[
	\beta \gtrsim \frac{k}{\sqrt{tn}}\sqrt{\log t}\,,
\]
so if $\beta = \eps \sqrt{d/n}$ for arbitrary constant $\eps$, we can choose large enough constant $t$ such that 
$\eps\sqrt{d} \gtrsim k\sqrt{\frac{\log t}{t}}$ 
and get an estimator that is highly correlated with $v$ or $-v$ in polynomial time $n\cdot d^{O(t)}$. 
Neither Covariance Thresholding nor the top eigenvector of the empirical covariance can work with small values of $\eps$, 
and limited brute force from \cite{DKWB19} requires quasi-polynomial time $n \cdot d^{O\Paren{\log d}}$ in these settings.

It is also interesting to compare our upper bound with the low degree polynomial lower bound from \cite{dKNS20}.
They showed that in the regime $k \le O\Paren{\sqrt{d}}$, 
polynomials of degree $D \le n /\log^2 n$ 
cannot distinguish\footnote{More precisely, they cannot \emph{strongly distinguish} 
	sequences of distributions in the sense of \cite{DBLP:journals/corr/abs-1907-11636}.} 
$\bm Y_1\ldots, \bm Y_n \sim N(0, \Id + \beta vv^\top)$ 
from  $\bm Y_1\ldots, \bm Y_n \sim N(0, \Id)$ if
\[
\beta \lesssim {\frac{k}{\sqrt{Dn}}\cdot {\log\Paren{2 + \frac{Dd}{k^2} }}}\,.
\]
and hence for such $\beta$ they cannot be used to design an estimator that is close to $v$ or $-v$ with high probability. 
Their lower bound does not formally imply that our upper bound is tight 
(that is, it does not imply that there are no better estimators than ours among low degree polynomials).
However, there is an interesting similarity between the lower bound and the upper bound: They have a very similar logarithmic factor. If this similarity can be formalized, it may lead to an algorithm that works in a small sample regime $n\ll d$.  
The term $d/n$ that we have in the logarithmic factor in the bound on $\beta$ is necessary for our techniques.
Many  other algorithms, like basic SDP or Covariance Thresholding, also have similar terms.
However, the low-degree lower bound does not have this term and 
\cite{dKNS20} provided an algorithm based on low degree polynomials that does not have such a term
and works as long as $\beta \gtrsim \frac{k}{\sqrt{n}} \sqrt{\log\Paren{2 + td/k^2}}$ even for very small $n$ (e.g. $n = d^{0.01}$). 
Finding an estimator with guarantees similar to ours in the small sample regime $n \ll d$ is an interesting open question, 
and low degree polynomials might be useful in designing such an estimator.

\paragraph{Adversarial perturbations}
Our approach also works in the presence of adversarial perturbations. 
\begin{theorem}[The Wishart model with adversarial perturbations]\label{thm:wishart-perturbations-results}
		Let $n,d,k,t\in \N$, $\beta > 0$, $\eps \in (0,1)$. 
	Let 
	$
	Y = \sqrt{\beta}\bm u v^\top + \bm W + E\,,
	$ 
	where $\bm u \sim N(0,1)^n$, $v\in \R^d$ is a $k$-sparse unit vector, $\bm W\sim N(0,1)^{n\times d}$ independent of $\bm u$ 
	and $E \in \R^{n\times d}$ is a matrix such that
	\[
	\Norm{E}_{1\to 2} \le \eps \cdot \min\Set{\sqrt{\beta}, \beta}\cdot \sqrt{n/k}\,,
	\]
	where $\Norm{E}_{1\to 2}$ is the maximal norm of the columns of $E$ and $\eps < 1$. 
	
	There exists an absolute constant $C > 1$, such that if 
	$n \ge C k$, 
	$k\ge C t\log^2 d$,
	\[
	\beta \ge C \frac{k}{\sqrt{tn}}\sqrt{\log\Paren{2 + \frac{td}{k^2}\Paren{1 + \frac{d}{n}} }}\,.
	\]
	and
	$
	\eps\sqrt{\log\Paren{1/\eps}} \le \frac{1}{C} \min\Set{1,\min\Set{\beta, \sqrt{\beta}} \cdot \sqrt{n/d}}\,,
	$
	then there exists an algorithm that, given $Y$, $k$ and $t$, 
	in time $n \cdot d^{O(t)}$ outputs a unit vector $\hat{\bm v}$ such that
	with probability $1-o(1)$ as $d\to \infty$,
	\[
	\abs{\iprod{\hat{\bm v}, v}} \ge 0.99\,.
	\]
\end{theorem}
To illustrate how large the adversarial perturbations are allowed to be, consider the following example: 
Let $k = \Theta\paren{\sqrt{d}}$, $n = \Theta\Paren{d}$, $\beta = \Theta(1)$, $t \le O(1)$. 
Then the columns of $E$ can have norm as large as $\Omega\Paren{\sqrt{k}}$. 
Note that in these settings column norms of $\sqrt{\beta}uv^\top$ can be $O\Paren{\sqrt{k}}$. 
Hence, in this regime, 
if we allow $E$ to be larger by a constant factor, the adversary can choose $E = -\sqrt{\beta}uv^\top$ and erase the signal.
As was shown in \cite{dKNS20}, in these settings Covariance Thresholding, Diagonal Thresholding
and the top eigenvector of the empirical covariance do not work with some perturbations $E$ such that $\Norm{E}_{1\to 2} \le k^{o(1)}$.

Our assumption on $E$ is stronger than the assumption from \cite{dKNS20}, 
which is $\Norm{E}_{1\to 2} \lesssim \min\Set{\beta, \sqrt{\beta}} \sqrt{n/k}$. 
Designing an estimator with guarantees similar to ours that works with larger $E$ is an interesting problem. 

Similar to the non-adversarial settings, our algorithms have the same guarantees\footnote{Assuming our bound on the columns of $E$.} 
as the sum-of-squares approach from \cite{dKNS20} 
if $k\le d^{1/2-\Omega(1)}$ and has asymptotically better guarantees if $k \ge d^{1/2-o(1)}$. Similarly to Covariance Thresholding in the non-adversarial case, in the regime $n \ge \Omega(d)$ and $k = \Theta\Paren{\sqrt{d}}$ basic SDP requires
$
\beta \ge c \sqrt{d/n}
$
for some specific constant $c$ and does not work for smaller $\beta$.
Our condition on $\beta$ in these settings is
$
\beta \gtrsim \frac{k}{\sqrt{tn}}\sqrt{\log t}\,,
$
so if $\beta = \eps \sqrt{d/n}$ for arbitrary constant $\eps$, we can choose large enough constant $t$ such that 
$\eps\sqrt{d} \gtrsim k\sqrt{\frac{\log t}{t}}$ 
and get an estimator that can be computed in polynomial time $n\cdot d^{O(t)}$. 
basic SDP cannot work with small values of $\eps$, 
and sum-of-squares approach from \cite{dKNS20} requires quasi-polynomial time in these settings.

\paragraph{The sparse planted vector problem}
As was observed in \cite{dKNS20}, sparse PCA with perturbations is a generalization not only for the spiked covariance model, 
but also for the planted sparse vector problem. 
In this problem we are given an $n$-dimensional subspace of $\R^d$ 
spanned by $n-1$ random vectors and a sparse vector, and the goal is to find the sparse vector.
More precisely, let $\bm g_1, \bm g_2,\ldots, \bm g_n$ be standard $d$-dimensional Gaussian vectors 
and let $\bm B$ be an $n\times d$ matrix whose first $n-1$ rows  are  $\bm g_1^\top,\ldots, \bm g_{n-1}^\top$ 
and the last row is a vector $\norm {\bm g_n} v^\top$, where $v\in \R^d$ is $k$-sparse and unit.
Let $\bm R$ be a random rotation of $\R^n$ independent of  $\bm g_1\ldots, \bm g_n$, and let $\bm Y = \bm R \bm B$.
The goal is to recover $v$ from $\bm Y$.

This problem can be seen as a special case    
of sparse PCA with perturbation matrix $E = -\frac{1}{\Norm{\bm u}^2}\bm u\bm u^\top \bm W $ (see \cref{planted-vector} for the proof).
Therefore, we can apply \cref{thm:wishart-perturbations-results} and get
\begin{corollary}[The sparse planted vector problem]\label{cor:planted-vector}
	Let $n,d,k,t\in \N$, $\beta > 0$. 
	Let 
	$
	\bm Y = \sqrt{\bm \beta}\bm u v^\top + \bm W - \frac{1}{\Norm{\bm u}^2}\bm u\bm u^\top \bm W \,,
	$ 
		where $\bm u \sim N(0,1)^n$, $v\in \R^d$ is a $k$-sparse unit vector, $\bm W\sim N(0,1)^{n\times d}$ independent of $\bm u$,
		and $\sqrt{\bm \beta} = \frac{\Norm{\bm u^\top \bm W}}{\Norm{\bm u}^2 }$.
	
	There exists an absolute constant $C > 1$, such that if 
			$d > n$, $n \ge C k$, 
		$k\ge C t\log^2 d$ and
	\[
	k \le \frac{1}{C} \cdot d \sqrt{t/n}\,,
	\]
	then there exists an algorithm that, given $\bm Y$, $k$ and $t$, 
in time $d^{O(t)}$ outputs a unit vector $\hat{\bm v}$ such that
with probability $1-o(1)$ as $d\to \infty$,
\[
\abs{\iprod{\hat{\bm v}, v}} \ge 0.99\,.
\]
\end{corollary}

Prior to this work, in the regime $n \ge \Omega(d)$,
polynomial time estimators were known only if
$k \le c d/\sqrt{n}$ for some small constant $c < 1$ 
(the existence of such an algorithm follows from 
Theorem 4.5 from \cite{dKNS20}).
We show that even if $k \ge 100d/\sqrt{n}$, 
sparsity can still be exploited and there are estimators that can be computed in polynomial time.

\paragraph{The Wigner model and symmetric noise}
Our techniques also work with sparse PCA in the Wigner model.
\begin{theorem}[The Wigner model]\label{thm:wigner-results}
	Let $k, d, t\in\N$, $\lambda > 0$. 	Let 
	$
	Y = \lambda vv^\top + \bm W + E\,,
	$ 
	where $v\in \R^d$ is a $k$-sparse unit vector and $\bm W\sim N(0,1)^{d\times d}$ and $E\in \R^{d\times d}$.

	There exists an absolute constant $C > 1$, such that if 
	$k\ge Ct\log d$,
	$
	\Normi{E} \le \frac{1}{C} \lambda / k\,,
	$ and
	\[
	\lambda \ge C k\sqrt{\frac{\log\Paren{2 + {td}/k^2 } }{t} }\,,
	\]
	then there exists an algorithm that, given $Y$, $k$ and $t$, 
	in time $ d^{O(t)}$ outputs a unit vector $\hat{\bm v}$ such that
	with probability $1-o(1)$ as $d\to \infty$,
	\[
	\abs{\iprod{\hat{\bm v}, v}} \ge 0.99\,.
	\]
\end{theorem}

Note that $E$ is allowed to be as large as possible (up to a constant factor). 
Similar to the Wishart model, 
previously known polynomial time algorithms required 
\[
\lambda \gtrsim \min\Set{k\sqrt{\frac{\log d}{t}}, k\sqrt{\log\Paren{2 + d/k^2}}, \sqrt{d}}\,,
\]
where $k\sqrt{\frac{\log d}{t}}$ corresponds to the limited brute force from \cite{DKWB19}, 
and $\min\Set{k\sqrt{\log\Paren{2 + d/k^2}}, \sqrt{d}}$ corresponds to basic SDP.
Similarly to the Wishart model, in the regime $k\ge d^{1/2-o(1)}$ we get asymptotically better guarantees than the algorithm from \cite{DKWB19}. 
In the regime $n \ge \Omega(d)$ and $k = \Theta\Paren{\sqrt{d}}$ basic SDP requires
$\lambda \ge c \sqrt{d}$
for some specific constant $c$ and does not work for smaller $\lambda$.
Our condition on $\lambda$ in these settings is
$
\lambda \gtrsim {k}\sqrt{\frac{\log t}{t}}\,,
$
so if $\lambda = \eps \sqrt{d}$ for arbitrary constant $\eps$, we can choose large enough constant $t$ such that 
$\eps\sqrt{d} \gtrsim k\sqrt{\frac{\log t}{t}}$ 
and get an estimator that can be computed in polynomial time $d^{O(t)}$. 
basic SDP cannot work with small values of $\eps$, 
and limited brute force requires quasi-polynomial time in these settings.

Our techniques can be also applied to more general model with symmetric noise studied in \cite{dNNS22}.
\begin{theorem}[Sparse PCA with symmetric heavy-tailed noise]\label{thm:wigner-symmetric-results}
	Let $k, d, t\in\N$, $\lambda > 0$.
Let 
$
\bm Y = \lambda vv^\top + \bm N\,,
$ 
where $v\in \R^d$ is a $k$-sparse unit vector such that $\normi{v}\le 100/\sqrt{k}$
and $\bm N$ is a random matrix with independent
(but not necessarily identically distributed) symmetric about zero entries\footnote{That is, $\bm N_{ij}$ and $-\bm N_{ij}$ have the same distribution.} such that for all $i,j \in [d]$,
$
\Pr\Brac{\Abs{\bm N_{ij}} \le 1} \ge 0.1\,.
$

There exists an absolute constant $C > 1$, such that if 
$k \ge C {t \log d}$,
	\[\
t \ge C\cdot  {\log\Paren{2+ d/k^2}} \,,
\]
and
\[
\lambda \ge k\,,
\]
then there exists an algorithm that, given $\bm Y$, $k$, $t$ and $\lambda$, 
in time $ d^{O(t)}$ outputs a unit vector $\hat{\bm v}$ such that
with probability $1-o(1)$ as $d\to \infty$,
\[
\abs{\iprod{\hat{\bm v}, v}} \ge 0.99\,.
\]

Moreover, we get the same guarantees if we are given only the upper triangle of $\bm Y$, i.e. the entries $\bm Y_{ij}$ such that $i<j$.
\end{theorem}

If $k = \Theta\Paren{\sqrt{d}}$, this algorithm runs in polynomial time as long as 
$\lambda \ge \eps\sqrt{d}$ for (arbitrary) constant $\eps$. 
It is the first known polynomial time algorithm for this model, since the algorithm from \cite{dNNS22} requires quasi-polynomial time. 
For example, our algorithm finds an estimator close to $v$ or $-v$ in polynomial time if $k = \sqrt{d}$ and $\lambda = k / 100$ 
when the noise has iid Cauchy entries\footnote{Cauchy noise is very heavy-tailed, the entries do not even have a finite first moment.} (with location $0$ and scale $1$), 
while prior to this work the fastest known algorithm in this regime required quasi-polynomial time even for standard Gaussian noise.

In the special case of Gaussian noise and $\lambda = k$, this algorithm matches the best known algorithmic guarantees.
Note, however, that in the regime $k \le d^{1/2 - \Omega(1)}$ this algorithm runs in quasipolynomial time. This is not surprising, since as was observed in \cite{dKNS20}, sparse PCA with symmetric noise actually generalizes  the planted clique problem.

More precisely, let $\bm G \sim G(d,1/2,k)$ be a random graph with a planted clique of size $k$. Let $\bm A$ be the adjacency matrix of $\bm G$. Let $J$ be the matrix with all entries equal to $1$ and let $\bm C = 2\bm A-J$. Note that the upper triangle of  $\bm C$ coincides with the upper triangle of $k \cdot vv^\top + \bm \eta$, where $\sqrt{k}\cdot v$ is the indicator vector of the vertices of the clique (so it is $k$-sparse), and $\bm \eta$ is the noise whose entries that correspond to the vertices of the clique are zero, and other entries are iid uniform over $\{\pm1\}$. 

The algorithm from \cref{thm:wigner-symmetric-results} solves the planted clique problem in time $n^{O\Paren{\log\Paren{2 + n/k^2}}}$, which matches the best known algorithmic guarantees for the planted clique\footnote{Up to a constant factor in the degree.}. Moreover,  for some $k = n^{\Omega(1)}$ it is conjectured to be impossible to solve it in time $n^{o(\log n)}$ (see \cite{planted_clique_conjecture} for more details). 
  Note that our algorithm  achieves best known algorithmic guarantees for both sparse PCA in the Wigner model and the planted clique problem\footnote{Assuming $\lambda = k$ for sparse PCA.}.
	\section{Techniques}
\label{sec:techniques}

The idea of our approach is similar to the well-known technique of reducing the constant in the planted clique problem. 
Recall that an instance of the planted clique problem is a random graph $\bm G$ sampled according to the following distribution: First, a graph is sampled from Erd\H{o}s-R\'enyi distribution $\cG(m, 1/2)$ (i.e. each pair of vertices is chosen independently to be an edge with probability $1/2$), 
and then a random subset of vertices of size $k$ is chosen (uniformly from the sets of size $k$ and independently from the graph)
and the clique corresponding to these vertices is added to the graph.
The goal is to find the clique. The problem can be solved in quasi-polynomial time, however, no polynomial time algorithm is known in the regime
$k \le o\Paren{\sqrt{m}}$.

\cite{AKS98} proposed a spectral algorithm that can be used to find the clique
in polynomial time if $k \gtrsim \sqrt{m}$.
They also introduced a technique that allows to find the clique in polynomial time if $k \ge \eps \sqrt{m}$ for arbitrary constant $\eps > 0$. 
The idea is to look at every subset $T$ of vertices of size $t \gtrsim \log\Paren{1/\eps}$ and consider the subgraph $\bm H(T)$ 
induced by the vertices of $\bm G$ that are adjacent to  $T$ (i.e. adjacent to every vertex of $T$). 
This subgraph has approximately $m' = 2^{-t} m \lesssim \eps^2 m$ vertices, and if $T$ was a part of the clique, then the clique is preserved in $\bm H(T)$, 
and since $k \gtrsim \sqrt{m'}$, 
we can find a clique applying the spectral algorithm to  $\bm H(T)$.
The running time of the algorithm is $m^{O(t)}$, so it is polynomial for constant $\eps$.

A similar  (but technically more challenging)  idea can be also used for sparse PCA. Recall that the instance of sparse PCA (in the Wishart model) is 
$
\bm Y = \sqrt{\beta}\bm u v^\top + \bm W\,,
$
where $\bm u \sim N(0,1)^n$, $v\in \R^d$ is a $k$-sparse unit vector, $\bm W\sim N(0,1)^{n\times d}$ independent of $\bm u$. 
To illustrate the idea, we assume that $v$ is flat, i.e. its nonzero entries are $\pm 1/\sqrt{k}$.
Instead of the adjacency matrix of the graph, we have the empirical covariance $\tfrac{1}{n} \bm Y^\top \bm Y$. 
For simplicity, let us ignore cross terms and assume that 
\[
\tfrac{1}{n} \bm Y^\top \bm Y \approx \frac{\norm{\bm u}^2}{n}\beta vv^\top + \tfrac{1}{n} \bm W^\top \bm W\,.
\]
Since $\bm u \sim N(0, \Id)$, $\norm{\bm u}^2 \approx n$. 
So we assume that we are given $\beta vv^\top + \tfrac{1}{n} \bm W^\top \bm W$, and the goal is to recover $v$. 
Similar to the planted clique, if $\beta \gtrsim \sqrt{d/n}$, there is a spectral algorithm for this problem (that computes the top eigenvector of the empirical covariance).

Suppose that $n\ge d$, $k \approx \sqrt{d}$ and $\beta = \eps \sqrt{d/n} \approx \eps k /\sqrt{n}$ for some constant $\eps>0$.
We can look at each subset $T$ of entries of size $t < k$ and try to find a (principal) submatrix of $\tfrac{1}{n} \bm Y^\top \bm Y$ that would be an
analogue of the graph $\bm H(T)$ from the algorithm for the planted clique. 
One option is to say that an entry $i$ is ``adjacent'' to $T$ if the sum of the elements in the $i$-th row of 
$\tfrac{1}{n} \bm Y^\top \bm Y$ is large. 
However, since $v$ has both positive and negative entries, the sum can be small even if here were no noise and $T\subset \supp(v)$. 
Hence we also need to take the signs of the entries of $v$ into account. 

Let $\cS_t$ be the set of all $t$-sparse vectors with entries from $\Set{0,\pm 1}$. 
Let us call $s \in \cS_t$ \emph{correct} if $\supp(s) \subset \supp(v)$ and for all nonzero $s_i$, $\sign(s_i) = \sign(v_i)$. If $s\in\cS_t$ is correct, then
\[
\Abs{\sum_{j \in \supp(s)} \beta v_i v_j s_j} = \beta\Abs{ v_i v^\top s} = \beta t/k\,.
\]

For $s\in \cS$ let us call an entry $i\in [d]$ \emph{adjacent} to $s$ if either $\Abs{\Paren{\tfrac{1}{n} \bm Y^\top \bm Ys}_i} \ge \beta t / (2k)$ or $i\in \supp(s)$, 
and let $\bm H(s)$ be a principal submatrix induced by indices adjacent to $s$. 
The size of $\bm H(s)$ is close to $pd$, 
where $p$ is the probability that $\Abs{\Paren{\tfrac{1}{n} \bm W^\top \bm Ws}_i}$ is greater than $\beta t / (2k)$.
We need to count $i\notin \supp(s)$ adjacent to $s$. 
The vector $\bm Ws$ has distribution $N(0, t\cdot \Id)$, hence $\norm{\bm Ws} \approx \sqrt{tn}$.
Since $i\notin \supp(s)$, the $i$-th row of $\bm W^\top$ is independent of $\bm Ws$, 
and the distribution of $\Paren{\tfrac{1}{n} \bm W^\top \bm Ws}_i$ is close to $N(0, t/n)$. By the tail bound for the Gaussian distribution,
\[
\Pr\Brac{\Abs{\Paren{\tfrac{1}{n} \bm W^\top \bm Ws}_i} \ge x \sqrt{t/n}} \le \exp\Paren{-x^2/2}\,.
\]
In our case, $x = \frac{\beta\sqrt{tn}}{2k}$. Hence for 
	\[
\beta \gtrsim \frac{k}{\sqrt{tn}}\sqrt{\log\Paren{td/k^2}}\,,
\]
we get $x \gtrsim \sqrt{\log\Paren{td/k^2}}$ and $p = \exp\Paren{-x^2/2} \lesssim {k^2}/{(td)}$. 
Therefore, $\bm H(s)$ has $d' \lesssim k^2/t$ entries. 
Moreover, by the same argument, $k' \approx \Paren{1-p}k \ge 0.999k$ entries of $v$ are adjacent to correct $s$, so the signal part of $\bm H(s)$ is 
close to $\beta vv^\top$.
Since for correct $s$ we get $\beta \gtrsim \sqrt{d'/n}$, we can try to use the spectral algorithm to recover the sparse vector from $\bm H(s)$.

Here we see the difference between planted clique and sparse PCA. In the planted clique problem, 
if we take a subset of the clique, we can easily recover the whole clique from the output of the spectral algorithm and we do not need to consider other sets after that.
In sparse PCA, since we do not know $\beta$ exactly, 
it might not be easy to understand if the observed $s$ was correct or not from the output of the spectral algorithm. 

We use the following observation: if we have computed the list $\bm L = \Set{\tilde{\bm v}(s)}$ of the top eigenvectors of $\bm H(s)$ for all $s\in \cS_t$, 
we can compute a vector close to $v$ (or to $-v$) from this list.
Indeed, if we erase all but the largest $O\Paren{k}$ entries (in absolute value) of the vectors from $\bm L$, 
we get a new list $\bm L'$ of $O(k)$-sparse vectors. 
It turns out that for correct $s$ not only $\tilde{\bm v}(s)$, but also the corresponding $O(k)$-sparse vector ${\bm v}'(s)\in \bm L'$ is close to $v$.
Moreover, for all $O(k)$-sparse unit vectors $x$ (in particular, for all vectors in $\bm L'$),
\[
x\Paren{\tfrac{1}{n}\bm W^\top \bm W}x = 1 \pm \tilde{O}\Paren{\sqrt{k/n}}\,.
\]
Hence we can just compute $\hat{\bm v} \in \argmax_{x\in \bm L'} x \Paren{\tfrac{1}{n}\bm Y^\top \bm Y}x$, and it is close to $v$, since
\[
x \Paren{\tfrac{1}{n}\bm Y^\top \bm Y}x \approx 1 + \beta \Iprod{v,x}^2 \pm \tilde{O}\Paren{\sqrt{k/n}}\,,
\]
and for $\beta > \frac{k}{\sqrt{tn}}$, the term $\tilde{O}\Paren{\sqrt{k/n}}$ is smaller than $\beta$ (as long as $t \ll k$).

Note that in the definition of adjacent entries we used $\beta$, which might be unknown. 
But that is not a problem: if we use some value $\beta / 2 < \beta' \le \beta$ instead of $\beta$, the algorithm still works. 
Hence we can use all possible candidates from $n^{-O(1)}$ to $n^{O(1)}$ 
such that one of them differs from $\beta$ by at most factor of $2$,
and in the end work not with the list $\bm L$, 
but with a list of size $O\Paren{\log n} \cdot \card{\bm L}$.

\begin{remark*}[Comparison with the Covariance Thresholding analysis from \cite{DBLP:conf/nips/DeshpandeM14}]\label{remark:covariance-thresholding}
	Our algorithm for $t = 1$ has running time $O(nd^2) + \tilde{O}\Paren{d^3}$. For $n\ge d$, 
	the running time is comparable to the running time of Covariance Thresholding $O(nd^2)$. 
	The guarantees of both algorithms are the same (up to a constant factor).
	One advantage of our algorithm is that it is much easier to analyze. 
	The crucial difference is that in Covariance Thresholding one has to bound the spectral norm of thresholded Wishart matrix, 
	which requires a sophisticated probabilistic argument. In our algorithm, 
	we only need to bound principal submatrices of the Wishart matrix, 
	and such bounds easily follow from concentration of the spectral norm of $\Paren{\tfrac{1}{n}\bm W^\top \bm W - \Id}$ 
	and a union bound argument. 
	Another advantage of our algorithms is that we can get better guarantees than Covariance Thresholding (by increasing $t$ and hence also the running time) 
	and use the same analysis for all $t$.
\end{remark*}

\begin{remark*}[Comparison with the algorithms from \cite{DKWB19}]\label{remark:limited-brute-force}
	The algorithms from \cite{DKWB19} also use vectors $s\in \cS_t$. 
	However, the crucial difference between our approaches is that they work with $s'\in \cS_t$ 
	that maximizes $s \Paren{\tfrac{1}{n}\bm Y^\top \bm Y}s$.
	This approach works only if $\beta \gtrsim \frac{k}{\sqrt{t}}\sqrt{\log d}$, 
	since for smaller $\beta$ the maximizer of $s \Paren{\tfrac{1}{n}\bm Y^\top \bm Y}s$ might be completely unrelated to $v$. 
	For our analysis it is not a problem, since the correct $s$ is only determined in the end from the list $\bm L'$.
\end{remark*}

\paragraph{Adversarial Perturbations}
Similar approach also works in the presence of adversarial perturbations, that is,
if the input is $Y = \sqrt{\beta}\bm u v^\top + \bm W + E$ 
such that the columns of $E$ have norm bounded by $b \ll \beta \sqrt{n/k}$.
This is interesting, since known algorithms for sparse PCA that use thresholding techniques and
are not based on semidefinite programming, like Diagonal Thresholding, Covariance Thresholding, 
or the algorithms from \cite{DKWB19}, do not work in these settings (see \cite{dKNS20} for more details).

As in the non-adversarial case, we can compute the submatrices $H(s)$ that are 
induced by indices adjacent to $s$, that is, indices $i$ such that 
either $\Abs{\Paren{\tfrac{1}{n}  Y^\top  Ys}_i} \ge \beta t / (2k)$ or $i\in \supp(s)$. 
Then, instead of computing the top eigenvector of $H(s)$, we compute 
$\tilde{X}(s) \in \argmax_{X \in \cP_k} \Iprod{X, H(s)}$, where
\[
\cP_k = \Set{X \in \R^{d\times d} \;\suchthat\; X\succeq 0\,, \Tr{X} = 1 \,, \Normo{X}\le k}
\]
is the feasible region of the basic SDP for sparse PCA. 
Then, we can compute the list $\bm L$ of top eigenvectors $\tilde{v}(s)$ of $\tilde{X}(s)$, 
and perform the same procedure as in the non-adversarial case to recover $\hat{v}$ from $\bm L$.

In order to show the correctness, we need to bound all terms of $\frac{1}{n}Y^\top Ys$. 
In adversarial settings cross terms can be large, and the most problematic term is $\frac{1}{n}E^\top \bm Ws$. 
Rows of $E^\top$ do not have large norm, but $\bm Ws$ has large norm $\Norm{\bm Ws} \approx \sqrt{tn}$,
and since $E$ can depend on $\bm W$, the entries of $E^\top \bm Ws$ can be large. 
In particular, for each correct $s$, the adversary can always choose $E$ such that the term 
$\frac{1}{n}E^\top \bm Ws$ is large enough to make $H(s)$ useless for recovering $v$.

To resolve this issue, we work with some probability distribution over the set of correct $s$ 
and show that $\frac{1}{n}E^\top \bm Ws$ has small expectation with respect to this distribution. 
In particular, it implies that for each $E$ there exists some $s'$ such that 
the term $\frac{1}{n}E^\top \bm Ws'$ is small\footnote{More precisely, this term has small norm, and it is enough for our analysis.}. 
For flat $v$, we just divide the support of $v$ into $m = k/t$ blocks of size $t$, and then each block corresponds to some correct $s\in \cS_t$.
Then it is enough to consider uniform distribution $\cU$ over the set $\Set{s_1,\ldots, s_m}$ of such $s$. 
By the concentration of spectral norm of $\bm W$, with high probability
\[
\frac{1}{n^2}\E_{\bm s \sim\cU} \Iprod{E_i, \bm W\bm s}^2 = \frac{1}{n^2m}\sum_{j=1}^m \Iprod{E_i, \bm W s_j}^2 
\le O\Paren{\frac{b^2t}{n^2m}\Paren{m +  n}} 
\le O\Paren{\frac{b^2t^2}{nk}} \ll \frac{\beta^2 t^2}{k^2}\,.
\]
Hence there exists some $s' \in \Set{s_1,\ldots, s_m}$ such that
\[
\Norm{\Paren{\frac{1}{n}E^\top \bm Ws'}_{\supp(v)}}^2 \ll \Norm{\beta vv^\top s'}^2\,.
\]
The other terms of $\frac{1}{n}Y^\top Ys$ can be bounded only assuming that $s'$ is correct 
(so it is not needed to use properties of $\cU$ anymore), 
hence the signal part of $H(s')$ is close to $\beta vv^\top$. 

In addition to $pd$ entries $i\in [d]\setminus \supp(v)$ adjacent to $s'$ that appear due to the term $\frac{1}{n}\bm W^\top \bm Ws$,
there could be some entries adjacent to $s'$ that appear from $\frac{1}{n}\Paren{E^\top Y + Y^\top E}s'$. 
We show that the number of such entries is at most $\eps^2\log(1/\eps) d$, 
where $\eps$ is the same as in \cref{thm:wishart-perturbations-results}.
Assuming our bound\footnote{
	This is the reason why our bound on $E$ 
	is worse than the bound from \cite{dKNS20}. } on the maximal norm of columns of $E$, 
we get $\eps^2\log(1/\eps) d \lesssim \beta n$,
and by standard properties of basic SDP for sparse PCA, the top eigenvector of $\tilde{X}(s')$ is close to $v$. 

To finish the argument, we need to show that we can still compute $\hat{v}$ close to $v$ or $-v$ from $\bm L$ 
even in the presence of perturbations. 
It is not hard since our argument depends only on the upper bound on 
$x \Paren{\frac{1}{n} Y^\top Y - \Id - \beta vv^\top}x$ for all $O(k)$-sparse $x$. 
As was shown in \cite{dKNS20}, our assumption on the maximal norm of columns of $E$ 
is enough to obtain the desired upper bound.

\paragraph{The Wigner model and symmetric noise} In the Wigner model the input is n
 $\bm Y = \lambda vv^\top + \bm W$. 
 The same argument as for the Wishart model works in these settings, 
 and the proof is technically simpler since there are no cross terms 
 and $\bm W$ is easier to analyze than $\frac{1}{n}\bm W^\top \bm W$ 
 that appears in the Wishart model. 
 Our approach for sparse PCA with perturbations also works for Wigner model, and
 the proof  is much easier in this case since the adversary cannot exploit magnitude 
 of the columns of $\bm W$.
 
The Wigner model with symmetric noise is more challenging.
 In these settings we assume that $\lambda$ is known.
 We cannot work with $\bm Ys$, since the noise is unbounded. 
 So we first threshold the entries of $\bm Y$. 
 Concretely, for $h > 0$ and $x\in \R$, let $\tau_{h}\Paren{x}$ be $x$ if $x\in [-h, h]$ and $\sign(x)\cdot h$ otherwise.
We apply this transformation for some $h = \Theta(\lambda / k)$ to the entries of $\bm Y$ and get a new matrix $\bm T$.
Then we use our approach for matrix $\bm T$ and for all $s\in \cS_t$ we compute the submatrices $\bm H(s)$. 
As long as $k \gtrsim d$, 
their algorithm applied to $\bm Y$ outputs a matrix that is close to $\lambda vv^\top$ in polynomial time. 
As in the Gaussian case, the submatrices $\bm H(s)$ have small size, 
and we can apply their result to every $\bm H(s)$ . 
However, since $\bm H(s)$ depends on $\bm Y$, 
the noise part of $\bm H(s)$ might not have the same distribution as  $\bm N$.
Fortunately, the error probability in \cite{dNNS22} is very small, which allows us to use union bound 
and conclude that for correct $s$, the output of the algorithm from \cite{dNNS22} on $\bm H(s)$ is close to $\lambda vv^\top$ . 
Moreover, since we know $\lambda$, 
we do not even need to work with the list of candidates in these settings: 
It is enough to check the norm of the output, and if it is close to $\lambda$, the output is close to $\lambda vv^\top$, 
and we can recover $v$ from it (up to sign).

	\section{The Wishart Model}
\label{sec:Wishart}

\paragraph{Notation} For $m_1,m_2\in \N$, we use the notation $\R^{m_1\times m_2}$ 
for the set of $m_1\times m_2$ matrices with entries from $\R$.
We denote by $N(0,1)^m$ an $m$-dimensional random vector with iid standard Gaussian entries. 
Similarly, we denote by $N(0,1)^{m_1\times m_2}$ an $m_1\times m_2$ random matrix with iid standard Gaussian entries. 
For $m\in \N$, we denote by $[m]$ the set $\Set{1,2,\ldots, m-1, m}$. 
For a vector $v\in \R^m$, we denote by $\norm{v}$ its $\ell_2$ norm, and for $p\in [1,\infty]$ we denote by $\norm{v}_p$ its $\ell_p$ norm. 
For a matrix $M\in \R^{m_1\times m_2}$, we denote by $\Norm{M}$ its spectral norm, and by $\Normf{M}$ its Frobenius norm. We write $\log$ for the logarithm to the base $e$.

$\,$

Recall that the input is an $n\times d$ matrix $\bm Y = \sqrt{\beta}\bm u v^\top + \bm W$, 
where $\bm u \sim N(0,1)^n$, $v\in \R^d$ is a $k$-sparse unit vector, $\bm W\sim N(0,1)^{n\times d}$ independent of $\bm u$. 
The goal is to compute a unit vector $\hat{\bm v}$ such that $\abs{\iprod{\hat{\bm v}, v}} \ge 0.99$ 
with high probability (with respect to the randomness of $\bm u$ and $\bm W$).

	First we define vectors $\bm z_{s}(r)$ that we will use in the algorithm. 
	Let $t\in \N$ be such that $1\le t \le k$ and let $\cS_t$ be the set of all $d$-dimensional vectors whose entries are in $\set{-1,0,1}$ 
	that have exactly $t$ nonzero coordinates.
	For $s\in \cS_t$ and $r > 0$ let  $\bm z_s\Paren{r}$  be the $d$-dimensional (random) vector defined as 
	
	\[
	 \bm z_{si}\Paren{r} =  
	 \begin{cases}
	 \ind{\Paren{\bm Y^\top \bm Y s}_i \ge r\cdot t\cdot n} \quad \text{if } s_i = 0\\
	 1 \qquad\qquad\qquad\; \text{otherwise}
	 \end{cases}
	\]
	
	The following theorem is a restatement of \cref{thm:wishart-classical-results}.
	\begin{theorem}\label{thm:wishart-classical-technical}
		Let $n,d,k,t\in \N$, $\beta > 0$, $0 < \delta < 0.1$. 
		Let $\bm Y = \sqrt{\beta}\bm u v^\top + \bm W$, 
		where $\bm u \sim N(0,1)^n$, $v\in \R^d$ is a $k$-sparse unit vector, $\bm W\sim N(0,1)^{n\times d}$ independent of $\bm u$.
		
Suppose that $n \gtrsim k + \frac{t\ln^2 d}{\delta^4}$,
$k\gtrsim {\frac{t\ln d}{\delta^2}}$ and
\[
\beta \gtrsim \frac{k}{\delta^2\sqrt{tn}}\sqrt{\ln\Paren{2 + \frac{td}{k^2}\cdot \Paren{1 + \frac{d}{n} }}}\,.
\]

	Then there exists an algorithm that, given $\bm Y$, $k$, $t$ and $\delta$, 
	in time $n \cdot d^{O(t)}$ outputs a unit vector $\hat{\bm v}$ such that
with probability $1-o(1)$ as $d\to \infty$,
\[
\abs{\iprod{\hat{\bm v}, v}} \ge 1 - \delta\,.
\]
	\end{theorem}
	
	\begin{lemma}\label{lem:wishart-erasing-probability}
		For all $s\in  \cS_t$ such that $s_i = 0$ and for all $\tau > 0$,
		\[
		\Pr\Brac{\Paren{\bm W^\top \bm Y s}_i \ge 
			\tau\cdot\Paren{\Norm{\bm  Ws} + \abs{\iprod{v,s}}\cdot\sqrt{\beta} \cdot \Norm{\bm u}} \;\given\; \Norm{\bm u}, \Norm{ \bm Ws}}
		\le \exp\Paren{-\tau^2 / 2}\,.
		\]
		\begin{proof}
			Denote $\bm g = \bm Ws \sim N\Paren{0, t\cdot \Id}$.
			\[
			\Paren{\bm W^\top \bm Y s}_i 
			= \Iprod{\bm W_i,\bm g} +  
			\sqrt{\beta}\cdot \Iprod{v,s}\cdot \Iprod{\bm W_i,\bm u}\,.
			\]
			Since $\bm W_i$, $\bm g$ and $\bm u$ are independent, 
			conditional distribution of $\Paren{\bm W^\top \bm Y s}_i$ 
		    given $\Norm{\bm u}$ and $\Norm{\bm g}$ is 
		    $N\Paren{0, \Norm{\bm g + \sqrt{\beta}\cdot  \Abs{\Iprod{v,s}} \cdot \bm u}^2}$.
		    The lemma follows from the triangle inequality and the tail bound for Gaussian distribution (\cref{fact:gaussian-tail-bound}). 
		\end{proof}
	\end{lemma}

	\begin{lemma}\label{lem:wishart-signal-preserved}
	Let $0 < \delta < 0.1$, $r \le  \frac{\beta \delta }{100k}$ 
	and suppose that $n\gtrsim t + \ln^2 d$, $\beta \gtrsim \frac{k}{\delta^2\sqrt{tn}}$ and
	$k \gtrsim \frac{\ln d}{\delta^2}$. 
	
	Let $s^* \in \argmax_{s\in \cS_t}\iprod{s, v}$.
	Then,\footnote{Here and further we denote by $a\circ b$ the entrywise product of vectors $a, b \in \R^d$.}
	\[
	\Norm{v\circ \bm z_{s^*}\Paren{r}}^2 = \iprod{v\circ \bm z_{s^*}\Paren{r}, v} \ge 1 - \delta\,,
	\]
with probability $1-d^{-10}$.
	\begin{proof}
		To simplify the notation we write $\bm z_{si}$  instead of $\bm z_{si}\Paren{r}$.
		Let $\bm g = \bm Ws^*$. 		 
		By \cref{fact:chi-squared-tail-bounds} with probability at least $1-\exp\paren{-n / 10}$,
		$n/2 \le \Norm{\bm u}^2\le 2{n}$ and $t n/2 \le \Norm{\bm g}^2\le  2tn$, 
		and in this proof we assume that these bounds on $\norm{\bm u}$ and $\norm{\bm g}$ are satisfied.
%		Let $\bm E_i = \E \Brac{\bm{z}_{s^*i} \given \Norm{\bm u}, \Norm{\bm g}}$
%		and
%		Then
%		\[
%		\Norm{v\circ \bm z_{s^*}}^2 = \sum_{i=1}^n \bm{z}_{s^*i} v_i^2 
%		\ge \sum_{i \in \cL} \bm{z}_{s^*i} v_i^2 = 
%		\sum_{i \in \cL} \bm E_i v_i^2 
%		+ \sum_{i \in \cL} \Paren{\bm{z}_{s^*i} - \bm E_i} v_i^2\,.
%		\]
%		
%		Let's bound the $\sum_{i \in \cL} \bm E_i v_i^2 $. 
%		Consider some $i\in \cL_\delta$. Note that
%		\[
%		\Paren{\bm Y^\top \bm Y s^*}_i 
%		= v_i \cdot \beta \Norm{\bm u}^2 \iprod{v,s^*} 
%		+ v_i \cdot \sqrt{\beta}\iprod{\bm u,\bm g}
%		+\Paren{\bm W^\top \bm Y s^*}_i\,,
%		\]
	
		 Let $T = \supp\paren{s^*}$. 
		 Note that it is the set of $t$ largest entries of $v$ (by absolute value).
Since $v$ is unit, there are two possibilities: either 
$\Norm{v_T}^2 > 1-\delta$, or $\Norm{v_T}^2 \le 1-\delta$. 
If $\Norm{v_T}^2 > 1-\delta$, then the statement is true since $\Norm{v\circ \bm z_{s^*}}^2 \ge \Norm{v_T}^2$.
If $\Norm{v_T}^2 \le 1-\delta$, 
then for every $j\in T$, $\abs{v_j} \ge {\sqrt{\delta}}/{\sqrt{k}}$. Hence
\[
\iprod{s^*, v} = \Normo{v_T} \ge \sqrt{\delta}\cdot  t/\sqrt{k}\,.
\]

		 Let $\cL = \Set{i \in [n] \suchthat \abs{v_i} \ge \frac{\sqrt{\delta}}{10\sqrt{k}}}$.
		  For all $i\in \cL$, $\beta \norm{\bm u}^2  \Iprod{v,s^*} \abs{v_i}\ge 10rtn$. Note that $\norm{v_\cL}^2 \ge 1 - \delta / 10$.

		 Let us write $\bm Y^\top \bm Ys^*$ as follows
		 \[
		 \bm Y^\top \bm Ys^*
		 = \beta \norm{\bm u}^2 \Iprod{v,s^*}  v+ \bm W ^\top\bm Ws^*
		 + \sqrt{\beta} v \bm u^\top \bm Ws^* +  \sqrt{\beta} \Iprod{v,s^*}  \bm W^\top \bm u \,.
		 \]
		 We will bound each term separately
		 
		 Consider the term $\bm W^\top \bm Ws^* = \bm W^\top\bm g$. By the Chi-squared tail bound (\cref{fact:chi-squared-tail-bounds}), 
		 with probability at least $1-\exp\paren{-\tau/2}$, 
		 \[
		 \Norm{\Paren{\bm W^\top\bm g}_{\cL\setminus\supp(s^*)}} \le 2\sqrt{tn\cdot \Paren{\Card{\cL} + \tau}}\,.
		 \]
		 Since $\Card{\cL} \le k$, with probability at least $1-\exp\paren{-k}$,
		 \[
		 \Norm{\Paren{\bm W^\top\bm g}_{\cL\setminus\supp(s^*)}} \le 10\sqrt{ktn} 
		 \le  \frac{\delta}{100}\Norm{\beta \norm{\bm u}^2 \iprod{v,s^*} \cdot v}\,,
		 \]
		 where we used $\beta \gtrsim \frac{k}{\delta^2\sqrt{nt}}$.
		 %Moreover, with probability at least $1-\exp\paren{-d}$,
		 %\[
		 %\Norm{\Paren{\bm W^\top \bm Ws^*}_{[d]\setminus\supp(s^*)}} \le 10\sqrt{dtn/\delta} \le  \frac{\delta \beta t n\sqrt{d}}{10\normo{v}}\,,
		 %\]
		 %where we used $\normo{v}\le \sqrt{k}$.
		 
		 Consider the term $\sqrt{\beta} \Iprod{v,s^*}  \bm W^\top \bm u$. 
		 Since $\bm W$ and  $\bm u$ are independent, by \cref{fact:chi-squared-tail-bounds} with probability at least $1-\exp(-k/10)$,
		 \[
		 \Norm{\Paren{\sqrt{\beta} \Iprod{v,s^*}  \bm W^\top \bm u}_{\cL}} \le 2\sqrt{\beta}\norm{\bm u} \sqrt{k} \cdot \iprod{v,s^*} 
		 \le \frac{\delta}{100}\Norm{\beta \norm{\bm u}^2 \iprod{v,s^*} \cdot v}\,,
		 \]
		 where we used $\beta n \gtrsim \frac{k\sqrt{n}}{\delta^2\sqrt{t}}$ and the fact that $n \ge t$.
		 %Moreover, with probability at least $1-\exp(-d/2)$,
		 %\[
		 %\Norm{\sqrt{\beta}  \bm W^\top \bm u v^\top s^*} 
		 %\le \frac{\delta\sqrt{d}}{\sqrt{k}}\Norm{\beta \norm{\bm u}^2 \iprod{v,s^*} \cdot v}\,.
		 %\]
		 
		 Consider the term $\sqrt{\beta} v \bm u^\top \bm Ws^*$. With probability at least $1-\exp(-0.1\sqrt{nt})$,
		 \[
		 \Norm{{\sqrt{\beta} v \bm u^\top \bm Ws^*}} 
		 \le 2\sqrt{\beta} \cdot \sqrt{t} \cdot \norm{\bm u}\cdot \paren{nt}^{1/4} \le  
		 \frac{\delta}{100}\Norm{\beta \norm{\bm u}^2 \iprod{v,s^*} \cdot v}\,,
		 \]
		 where we used the fact that the distribution of $\bm u ^\top\bm Ws^*$ given $\norm{\bm u}$ is $N(0, t\cdot \norm{\bm u}^2)$.
		 
		 By \cref{lem:small-perturbation-thresholding}, 
		 \[
		 \norm{v_{\cL} \circ \bm z_{s^*}}^2 \ge \Paren{1-\delta/10} \norm{v_\cL}^2 \ge 1 - \delta\,.
		 \]
	\end{proof}
\end{lemma}

\begin{lemma}\label{lem:wishart-number-of-nonzero-entries}
	Let $r > 0$ and $s\in \cS_t$. 
	With probability at least $1-\delta$, number of entries $i$ such that
	\[
	\Abs{\Paren{\bm W^T Y s}_i} \ge t\cdot r \cdot n
	\]
	is bounded by $pd + 2\sqrt{pd \ln \Paren{1/\delta}}$, where $p = \exp\Paren{-tnr^2/2}$.
\end{lemma}
\begin{proof}
	By \cref{lem:wishart-erasing-probability} and Chernoff bound \cref{fact:chernoff},
	number of such entries is at most $pd + \tau d$ 
	with probability at least $1-\exp\Paren{-\frac{\tau^2 d}{2p}}$. With $\tau = 2\sqrt{\ln\Paren{1/\delta}p/d}$ we get the desired bound.
\end{proof}

\begin{lemma}\label{lem:wishart-spectral-bound}
	For $s\in \cS_t$, let ${\bm N}(s)=\Paren{\bm Y^\top \bm Y - n\cdot \Id - \beta \Norm{\bm u}^2 v^\top v}\circ \Paren{\bm z_{s}\Paren{r}\bm z_{s}^\top\Paren{r}}$ and let $p = \exp\Paren{-tnr^2/2}$.
	Suppose that $k\ge \ln d$.
	
	Then for each $s\in \cS_t$, with probability $1-2d^{-10}$,
	\[
	\Norm{{\bm N}(s)} \le 10 \sqrt{\Paren{n+\beta n}\cdot \Paren{pd + k} \ln\Paren{\frac{ed}{pd + k}}}
	+ 10 \Paren{pd + k} \ln\Paren{\frac{ed}{pd + k}}
	\]
\end{lemma}
\begin{proof}
		To simplify the notation we write $\bm z_{si}$  instead of $\bm z_{si}\Paren{r}$. 
 By \cref{lem:wishart-number-of-nonzero-entries},
	number of nonzero $\bm z_{si}$ is at most  
 $2pd + 2\sqrt{pd \ln \paren{d}} + 2k \le 4pd  + 4k$ with probability at least $1-d^{-10}$.
	Applying \cref{lem:spectral-bound},
	we get the desired bound.
\end{proof}

\begin{lemma}\label{lem:spectral-bound-technical}
	Let $0 < \delta < 0.1$ and  $\frac{\delta\beta}{200k}\le r \le \frac{\delta\beta}{100k}$. 	Let $p =\exp\paren{-tnr^2 / 2}$.
	Suppose that $n \ge t\ln^2 d$, $k \ge t \ln d$ and
	\[
	\beta \gtrsim \frac{k}{\delta^2\sqrt{tn}}\sqrt{\ln\Paren{2 + \frac{td}{k^2}\cdot \Paren{1 + \frac{d}{n} }}}\,.
	\]
	Then 
	\[
	\sqrt{\Paren{n+\beta n}\cdot \Paren{pd + k} \ln\Paren{\frac{ed}{pd + k}}}  +\Paren{pd + k} \ln\Paren{\frac{ed}{pd + k}}
	\le \delta \beta n / 10\,.
	\]
\end{lemma}

\begin{proof}
	First note that since
	\[
	\ln\Paren{2 + \frac{td}{k^2} + \frac{td^2}{k^2n}} \ge 0.5\cdot \ln\Paren{2 + \frac{td}{k^2} + \sqrt{\frac{td^2}{k^2n}}}\,,
	\]
	we get
	\[
	\beta \gtrsim \frac{k}{\delta^2\sqrt{tn}}\sqrt{\ln\Paren{2 + \frac{td}{k^2}\cdot \Paren{1 + \frac{k}{\sqrt{nt} } }}}\,.
	\]
	
	Since
	\[
	\sqrt{\Paren{n+\beta n}\cdot \Paren{pd + k} \ln\Paren{\frac{ed}{pd + k}}} \le \sqrt{\Paren{n+\beta n}\cdot pd\ln\Paren{\frac{e}{p}}} +\sqrt{\Paren{n+\beta n}\cdot k\ln\Paren{\frac{ed}{pd + k}}}\,,
	\]
	we can bound the term with $pd$ and the term with $k$ separately. 
	Let us bound the first term. Note that 
	\[
	\exp\paren{-tnr^2 / 4}\le  \exp\Brac{-\frac{1}{\delta^2} \ln\Paren{2 + \frac{td}{k^2}\cdot \Paren{1 + \frac{k}{\sqrt{nt}} } }}
	\le \frac{\delta^2 k}{\sqrt{td + kd\sqrt{t/n}}} \lesssim \min\set{\beta,\sqrt{\beta}}\cdot \delta^2  \sqrt{n/d} \,.
	\]
	Hence
	\[
	\Paren{\sqrt{n}+\sqrt{\beta n}}\cdot\Paren{\exp\Paren{-\frac{tnr^2}{2}} tnr^2\sqrt{dn}}\le 
	\Paren{\sqrt{n}+\sqrt{\beta n}}\cdot\Paren{\exp\Paren{-\frac{tnr^2}{4}} \sqrt{dn}}
	\le 0.01 \delta^2\beta n\,.
	\]
	
	Since $t\le \frac{k}{\ln\Paren{ed/k}}$ and $n \ge k\ln\Paren{\frac{ed}{k}}$, the second term can be bounded as follows
	\[
	\sqrt{\Paren{n+\beta n}\cdot k\ln\Paren{\frac{ed}{pd + k}}} \le
	\sqrt{n k\ln\Paren{\frac{ed}{k}}} + \sqrt{\beta n k\ln\Paren{\frac{ed}{pd + k}}} \le 0.01 \delta\beta n\,.
	\]
	
	Now, let us bound
	\[
	\Paren{pd + k} \ln\Paren{\frac{ed}{pd + k}} \le pd\ln(e/p) + k\ln(ed/k)\,.
	\]
	The first term can be bounded as follows:
	\[
	 pd\ln(e/p) \le d \cdot \frac{k^2}{td + kd\sqrt{t/n}} \le k\sqrt{n/t}  \le 0.01 \delta \beta n\,.
	\]
	For the second term,
	\[
	k\ln(ed/k) \le k\sqrt{n/t} \le 0.01 \delta \beta n\,.
	\]
\end{proof}
	
\begin{lemma}\label{lem:wishart-correlation-with-argmax}
Let $0 < \delta < 0.1$ and  $\frac{\delta\beta}{200k}\le r \le \frac{\delta\beta}{100k}$. 
For $s\in \cS_t$ let $\hat{\bm v}(s)$ be the top\footnote{A unit eigenvector that corresponds to the largest eigenvalue} 
eigenvector of $\Paren{\bm Y^\top \bm Y}\circ\Paren{\bm z_s\Paren{r} \bm z^\top_s\Paren{r}}$.

Suppose that $n \gtrsim \frac{t\ln^2 d}{\delta^4}$,
$k\gtrsim {\frac{t\ln d}{\delta^2}}$ and
\[
\beta \gtrsim \frac{k}{\delta^2\sqrt{tn}}\sqrt{\ln\Paren{2 + \frac{td}{k^2}\cdot \Paren{1 + \frac{d}{n} }}}\,.
\]

Then there exists $s'\in \cS_t$ such that with probability $1-10d^{-10}$,
\[
\abs{\iprod{\hat{\bm v}(s'), v}} \ge 1 - 20\delta\,.
\]
\end{lemma}

\begin{proof}	
	Let $s' \in \argmax_{s\in \cS_t}\iprod{s, v}$ 
	and let $\tilde{\bm v} = v \circ \bm z_{s'}(r)$.
	Then
	\[
	\Paren{\bm Y^\top \bm Y}\circ\Paren{\bm z_{s'}\Paren{r} \bm z^\top_{s'}\Paren{r}}= \beta \Norm{\bm u}^2\tilde{\bm v} \tilde{\bm v}^\top 
	+ {\bm N}(s')\,,
	\]
	where ${\bm N}(s')$ is the same as in \cref{lem:wishart-spectral-bound}. 
	By \cref{lem:spectral-bound-technical} and \cref{lem:wishart-spectral-bound}, with probability $1-2d^{-10}$,
	\[
	\Norm{\bm N(s')} \le\delta {\beta n}\,.
	\]
	
	Since with probability at least $1-\exp(-n/10)$, $\Norm{u}^2 \ge n/2$, we can use standard (non-sparse) PCA. Concretely, by \cref{lem:maximizer-lemma},
	$\abs{\iprod{\hat{\bm v}(s'), \tilde{\bm v}}} \ge \norm{\tilde{v}}^2 - 2\delta$.
	By \cref{lem:wishart-signal-preserved}, with probability $1-10d^{-10}$, $\norm{\tilde{v}}^2 = \iprod{\tilde{\bm v}, v} \ge 1 - \delta$. 
	Using \cref{fact:three-close-vectors} we get the desired bound.
\end{proof}

\begin{proof}[Proof of \cref{thm:wishart-classical-technical}] 
Since $\sqrt{n/2} \le \Norm{\bm u} \le\sqrt{2n}$ with probability at least $1-\exp\paren{-n/10}$, 
in this proof we assume that this bound on $\Norm{\bm u}$ holds.

Let $\tilde{\delta} = \delta / 100$.
Note that we can apply \cref{lem:wishart-correlation-with-argmax} if $\frac{\tilde{\delta} \beta}{200k}\le r \le \frac{\tilde{\delta} \beta}{100k}$.
Since we do not know $\beta$, we can create a list of candidates for $r$ of size at most $2\ln n$ 
(from $r = 1/n$ to $r = 1$).

For all $s\in \cS_t$ and for all candidates for $r$ we compute $\hat{\bm v}(s)$ as in \cref{lem:wishart-correlation-with-argmax}. By \cref{lem:wishart-correlation-with-argmax},  we get a list of vectors $\bm L$ of size $2\ln (n)\cdot \card{\cS_t}$ such that with probability $1-20\ln(n)d^{-10} \ge 1 - d^{-9}$  there exists $\bm v^* \in \bm L$ such that $\abs{\iprod{\bm v^*, v}} \ge 1 - 20\tilde{\delta} \ge 1 - \delta / 5$. 

By  \cref{fact:k-sparse-norm-gaussian} and \cref{lem:spectral-norm-cross}, %and union bound over all $k' \in [d]$,
$k'$-sparse norm of $\bm Y^\top \bm Y - n\Id - \beta\norm{\bm u}^2 vv^\top$ is bounded
by
\[
10\sqrt{nk'\ln\Paren{ed/k'}}+ 10\sqrt{\beta n k'\ln\Paren{ed/k'}}
\]
with probability at least $1-2d^{-9}$.
Let us show that if  $k' \lesssim \frac{\delta^2\beta^2 n}{\Paren{1+\beta}\ln d}$, then this $k'$-sparse norm is bounded by $\delta \beta n / 10$.
 Indeed, if $\beta \ge 1$, then
\[
\frac{\delta^2\beta^2 n}{\Paren{1+\beta}\ln d} \ge \frac{\delta^2\beta n}{2\ln d} \gtrsim k\sqrt{\frac{n}{t \ln^2 d}} \gtrsim k/\delta^2\,,
\]
and if $\beta < 1$,
\[
\frac{\delta^2\beta^2 n}{\Paren{1+\beta}\ln d} \ge \frac{\delta^2\beta^2 n}{2\ln^2 d} \gtrsim \frac{k^2}{\delta^2 t \ln d} \gtrsim k/\delta^2\,.
\]

Applying \cref{lem:list-decoding} with $k' = \lceil100k/\delta^2 \rceil$,
we can compute a unit vector $\hat{\bm v}$  such that
\[
\Abs{\Iprod{\hat{\bm v}, v}}\ge 1 - \delta\,.
\]
%In order to determine correct $k'$, we compute $\Norm{\bm Y \hat{\bm v}(k')}^2$ and the Frobenius norm of $\Paren{\bm Y^\top \bm Y}_{\cK' \times \cK'}$, where $\cK'$ is the support of $\hat{\bm v}(k')$. For $k' \ge \ln^2 d$, with probability $1-d^{-10}$, \Gnote{Frobenius norm of Wishart matrix}
%\[
%0.1 k'\sqrt{n} - \beta \Norm{\bm u}^2\le \Normf{\Paren{\bm Y^\top \bm Y}_{\cK' \times \cK'}} \le 10k'\sqrt{n\ln d} + \beta \Norm{\bm u}^2\,.
%\]
%Note that if $k/\delta^2 \lesssim  k' \lesssim \frac{\delta^2\beta^2 n}{\Paren{1+\beta}\ln^2 d}$, then
%\[
%\Normf{\Paren{\bm Y^\top \bm Y}_{\cK' \times \cK'}}  
%\le \min\Set{\frac{\Norm{\bm Y \hat{\bm v}(k')}^2}{\sqrt{n}}, \sqrt{n} \Norm{\bm Y \hat{\bm v}(k')}}\cdot \Paren{\frac{\delta}{100 {\ln d }}}^2\,.
%\]
%Moreover, the largest $k'$ that satisfies this bound is correct. Indeed, otherwise  
%we get $\beta \Norm{\bm u}^2 \le 2\beta n \le \frac{20}{\delta}\Paren{1+\sqrt{\beta}}\sqrt{k'n\ln(ed/k')} \le 0.01k'\sqrt{n}$ (since $k'\gtrsim k/\delta^2$ and $n\ge \frac{k\ln^2 d}{\delta^4}$)
%and it is a contradiction:
%\[
%0.09 k'\sqrt{n}
%\le \Normf{\Paren{\bm Y^\top \bm Y}_{\cK' \times \cK'}} \le 0.02\cdot k'\sqrt{n}\,,
%\]
%Therefore, the corresponding  $\hat{\bm v}(k')$ is the desired estimator.
\end{proof}
	\section{Adversarial Perturbations}
\label{sec:perturbations}

\paragraph{The sparse planted vector problem}\label{planted-vector} Let us show that the sparse planted vector problem is a special case 
of sparse PCA with perturbation matrix $E = -\frac{1}{\Norm{\bm u}^2}\bm u\bm u^\top \bm W $ . Let $\bm Y = \bm R \bm B$. 
It follows that
\[
\bm Y = \norm{\bm g_n} \bm R_n v^\top + \sum_{i=1}^{n-1} \bm R_i \bm g_i^\top\,.
\]
Note that $\sum_{i=1}^{n-1} \bm R_i \bm g_i^\top$ 
has the same distribution as $\Paren{\Id - \bm R_n \bm R_n^\top}\bm W$, where $\bm W\sim N(0,1)^{n\times d}$ is independent of $\bm R$.
Hence for Gaussian vector $\bm u$ such that $\frac{1}{\norm{\bm u}}\bm u = \bm R_n$, we get
\[
\bm Y =\frac{\norm{\bm g_n}}{\norm{\bm u}} \bm u v^\top + \bm W - \frac{1}{\Norm{\bm u}^2}\bm u\bm u^\top \bm W \,.
\] 
With high probability $\beta := \frac{\norm{\bm g_n}^2}{\norm{\bm u}^2} = \Paren{1+o(1)} d/n$.
Note that columns of $-\frac{1}{\Norm{\bm u}^2}\bm u\bm u^\top \bm W$ have norm at most $10\sqrt{\log d}$ with high probability.
In these settings, $\eps$ from \cref{thm:wishart-classical-results} is allowed to be as large as $\Omega(1)$, and we get a bound $\Norm{E}_{1\to 2} \lesssim \sqrt{d/k}$. Hence for $d \gtrsim k \log d$, $E$ is allowed to have columns of norm  $10\sqrt{\log d}$.

\paragraph{The Wishart model with perturbations} The following theorem is a restatement of \cref{thm:wishart-perturbations-results}.
\begin{theorem}\label{thm:wishart-perturbations-technical}
	Let $n,d,k,t\in \N$, $\beta > 0$, $0 < \delta < 0.1$. 
	Let 
	\[
	\tilde{Y} = \sqrt{\beta}\bm u v^\top + \bm W + E\,,
	\] 
	where $\bm u \sim N(0,1)^n$, $v\in \R^d$ is a $k$-sparse unit vector, $\bm W\sim N(0,1)^{n\times d}$ independent of $\bm u$ 
	and $E \in \R^{n\times d}$ is a matrix such that
	\[
	\Norm{E}_{1\to 2} = b\le \eps \cdot \min\set{\sqrt{\beta}, \beta}\cdot \sqrt{n/k}\,,
	\]
	where $\Norm{E}_{1\to 2}$ is the maximal norm of the columns of $E$ and 
	$\eps\sqrt{\ln\Paren{1/\eps}} \lesssim \delta^6\min\Set{1,\min\Set{\beta, \sqrt{\beta}} \cdot \sqrt{n/d}}$.
	
Suppose that $n \gtrsim k + \frac{t\ln^2 d}{\delta^4}$,
$k\gtrsim {\frac{t\ln d}{\delta^2}}$ and
	\[
	\beta \gtrsim \frac{k}{\delta^6\sqrt{tn}}\sqrt{\ln\Paren{2 + \frac{td}{k^2}\Paren{1 + \frac{d}{n}} }}\,.
	\]
	
	Then there exists an algorithm that, given $\bm Y$, $k$ and $t$, 
	in time $n \cdot d^{O(t)}$ outputs a unit vector $\hat{\bm v}$ such that
	with probability $1-o(1)$ as $d\to \infty$,
	\[
	\abs{\iprod{\hat{\bm v}, v}} \ge 1-\delta\,.
	\]
\end{theorem}
%Before proving the theorem, note that $\ln\Paren{1/\eps} \le \ln\Paren{2 + \frac{td}{k^2}\cdot \Paren{1 + \frac{k}{\sqrt{nt} } }}$. 
%As was shown in the proof of \cref{lem:spectral-bound-technical}, 
%	\[
%\beta \gtrsim \frac{k}{\delta^2\sqrt{tn}}\sqrt{\ln\Paren{2 + \frac{td}{k^2}\cdot \Paren{1 + \frac{k}{\sqrt{nt} } }}}\,.
%\]
%Hence $\beta \gtrsim \frac{k}{\delta^2\sqrt{tn}}\sqrt{\ln\Paren{1/\eps}}$, 
%and our condition on $E$ implies 	
%\[
%\Norm{E}_{1\to 2} \lesssim \min\Set{\beta, 1} \cdot \sqrt{\frac{kn}{dt}}\,.
%\]

\begin{lemma}\label{lem:perturbations-column-norm-of-noise}
	Let $\cL \subset [d]$ be the set of indices of the largest (in absolute value) $\ell$  entries of $v$, where $\ell = \min\Set{\lceil \normo{v}^2/\delta^2\rceil, k}$.
	Suppose that $\beta \gtrsim \frac{k}{\delta^6\sqrt{nt}}$ and $n \gtrsim k + t \ln^2 d$, $k \gtrsim \ln d$.
	
	Then
	with probability $1-d^{-10}$ there exists $s' \in \cS_t$ such that either
	\[
	\Norm{\Paren{\Paren{\tilde{Y}^\top\tilde{Y} -  \beta \norm{\bm u}^2 vv^\top}s'}_{\cL\setminus \supp(s')}}  
	\le 10\delta \cdot \Norm{\beta \norm{\bm u}^2 vv^\top s'} \,,
	\]
	or 
	\[
	\Norm{v_{\supp (s')}}\ge 1-\delta.
	\]
	
	Moreover, 
	\[
	\Norm{\Paren{E^\top\bm{Y} + \bm{Y}^\top E}s'}
	\le 10^4\cdot \frac{\eps}{\delta^2} \cdot\sqrt{d} \cdot  \frac{\beta n t}{\sqrt{k}\cdot\normo{v}}\,.
	\]
\end{lemma}
\begin{proof}
First note that
\[
\normo{v_\cL} \ge \normo{v} - \delta \ge \Paren{1-\delta}\normo{v}\,.
\]
and
\[
\norm{v_\cL} \ge \norm{v} - \normo{v_{[d]\setminus \cL}} \ge \norm{v} - \delta \ge \Paren{1-\delta}\norm{v}\,.
\]
Also note that for all $i\in \cL$, $\abs{v_i} \ge \frac{\delta}{\normo{v}}$.

We will define a distribution over $s \in \cS_t$, and then we show via probabilistic method that $s'$ with desired properties exists by bounding terms of $\tilde{Y}^\top\tilde{Y}s$ one by one. 
Also, with probability at least $1-\exp(n/2)$, $2\sqrt{n} \ge \norm{u} \ge \sqrt{n}/2$ 
and in the proof we will always assume that $2\sqrt{n} \ge \norm{u} \ge \sqrt{n/2}$.

Consider the partition of $\cL$ into $m = \lceil \card{\cL}/t \rceil$ disjoint blocks $b_1, \ldots, b_m$ 
of size\footnote{If the last block has smaller size, we can add arbitrary entries to it.} $t$.
Each block $b_j$ corresponds to $s(j)\in\cS_t$ such that $s(j)_i = 0$ for $i\notin b_j$ and $s(j)_i = \sign(v_i)$ for $i\in b_j \cap \supp\paren{v}$.

If $\card{\cL} \le t$, then $\Norm{v_{\supp(s(1))}}\ge \Norm{v_{\cL}} \ge 1-\delta$.

If $\Card{\cL} > t$, consider the uniform distribution $\cU$ over $s(j)$. We get
\[
\E_{\bm s\sim \cU}  \Iprod{v, \bm s} = \frac{1}{m} \sum_{j=1}^m\Normo{v\circ s(j)} =\frac{1}{m}  \Normo{v} 
\ge \frac{t}{2\cL}\Normo{v} \ge \frac{\delta^2t}{2\normo{v}} \ge \frac{\delta^2t}{2\sqrt{k}}\,.
\]
Moreover, 
\[
\E_{\bm s\sim \cU}  \Iprod{v, \bm s} = \frac{1}{m}  \Normo{v} \le \frac{t}{\cL}\Normo{v} \le \frac{t}{\normo{v}}\,.
\]

Let us write $\tilde{Y}^\top\tilde{Y}$ as follows:
\[
\tilde{Y}^\top\tilde{Y} 
= \beta \norm{\bm u}^2 vv^\top + \bm W ^\top\bm W
+ \sqrt{\beta} v \bm u^\top \bm W +  \sqrt{\beta}  \bm W^\top \bm u v^\top 
+ E^\top E 
+ \sqrt{\beta} v \bm u^\top E + \sqrt{\beta}  E^\top \bm u v^\top 
+ \bm W^\top E + E^\top \bm W\,.
\]
We will bound each term separately.

Consider the term $E\bm W^\top$.
\[
\E_{\bm s\sim \cU} \Abs{\Iprod{E_i, \bm W\bm s}}^2
= \frac{1}{m}\sum_{j = 1}^m \Abs{\Iprod{E_i, \bm Ws(j)}}^2
\]
Since for different $j_1$ and $j_2$, $s(j_1)$ and $s(j_2)$ have disjoint supports, 
$\bm Ws(j)\sim N(0,t)^n$ are independent. By the concentration of spectral norm of Gaussian matrices (\cref{fact:bound-covariance-gaussian}),
\[
\sum_{j=1}^m \Iprod{E_i, \bm Ws(j)}^2 \le 4t\Paren{m + n} \cdot \Norm{E_i}^2 \le 10tn b^2\,
\]
with probability at least $1-\exp(-n)$. Hence
\[
\E_{\bm s\sim \cU} \Abs{\Iprod{E_i, \bm W\bm s}}^2 \le \frac{10tb^2 n}{m} \,.
\]
Therefore, with probability at least $1-d\exp(-n)$,
\[
\E_{\bm s\sim \cU} \sum_{i\in \cL}\Abs{\Iprod{E_i, \bm W\bm s}}^2 \le \card{\cL} \frac{10tb^2 n}{m} 
\]
and
\[
\E_{\bm s\sim \cU} \sum_{i\in [d]}\Abs{\Iprod{E_i, \bm W\bm s}}^2 \le d \frac{10tb^2 n}{m} 
\]
Hence with probability at least $1-d\exp(-n/2)$ there exists $s' = s(j)$ such that 
\[
\frac{10t}{\normo{v}} \ge \Iprod{v, s(j)} \ge \frac{\delta^2t}{10\normo{v}} \ge \frac{\delta^2t}{10\sqrt{k}}\,,
\]
and
\[
\sum_{i\in\cL} \Abs{\Iprod{E_i, \bm W s(j)}}^2
\le\card{\cL} \frac{40tb^2 n}{m} \le 40t^2\eps^2\beta^2 n^2/k \le  \Paren{\frac{100 \eps}{\delta^2}}^2  \Norm{\beta \norm{\bm u}^2 vv^\top s(j)}^2
\le  {\delta^2}{}\Norm{\beta \norm{\bm u}^2 vv^\top s(j)}^2 \,,
\]
where we used $\card{\cL} \le mt$, and
\[
\sum_{i\in [d]}\Abs{\Iprod{E_i, \bm W\bm s'}}^2 \le  
\Paren{\frac{100 \eps}{\delta^2}}^2 \cdot\frac{d}{k} \cdot\Norm{\beta \norm{\bm u}^2 vv^\top s(j)}^2\,.
\]

Consider the term $\bm W^\top E$. By \cref{fact:bound-covariance-gaussian}, with probability at least $1-\exp(-n/2)$,
\[
\Norm{\bm W^TEs(j)} \le tb \Norm{\bm W^\top} \le 10tb\sqrt{n}\le 10t \eps\beta n / \sqrt{k}
\le \frac{400\eps}{\delta^2}\Norm{\beta \norm{\bm u}^2 \iprod{v,s(j)} \cdot v} 
\le {\delta}\Norm{\beta \norm{\bm u}^2 vv^\top s(j)} \,.
\]

Consider the term $\sqrt{\beta}  E^\top \bm u v^\top$:
\[
\Norm{\Paren{\sqrt{\beta}  E^\top \bm u v^\top s(j)}_{\cL\setminus \supp(s')} }
\le \sqrt{k} \Normi{\sqrt{\beta}  E^\top \bm u v^\top s(j)}
\le \sqrt{k\beta} \cdot b \cdot \norm{\bm u}\cdot \iprod{v,s(j)} 
\le  10\eps\Norm{\beta \norm{\bm u}^2 \iprod{v,s(j)} \cdot v}\,.
\]

Consider the term $ \sqrt{\beta} v \bm u^\top E $:
\[
\Norm{\sqrt{\beta} v \bm u^\top E s(j)} \le\sqrt{\beta} tb\norm{\bm u} 
\le \frac{\eps}{\delta^2}\cdot \frac{\delta^2 t}{\sqrt{k}}\cdot \beta \sqrt{n} \norm{\bm u} 
\le\frac{10\eps}{\delta^2}\Norm{\beta \norm{\bm u}^2 \iprod{v,s(j)} \cdot v}\,.
\]

Consider the term $E^\top E $:
\[
\Norm{\Paren{E^\top Es(j)}_{\cL\setminus \supp(s')}} 
\le \sqrt{k} \cdot\Normi{E^\top Es(j)} 
\le \sqrt{k}\cdot b^2t 
\le \sqrt{k}\cdot\frac{\eps^2}{\delta^2}\cdot \frac{\delta^2t}{k}\beta n 
\le  {10\eps}\Norm{\beta \norm{\bm u}^2 \iprod{v,s(j)} \cdot v}\,.
\]

Moreover, note that from the bounds above we also get
\[
\Norm{\Paren{E^\top\bm{Y} + \bm{Y}^\top E}s(j)}
\le 1000\cdot{\frac{\eps}{\delta^2}} \cdot \sqrt{\frac{d}{k}} \cdot \Norm{\beta \norm{\bm u}^2 vv^\top s(j)}\le 
10^4 \frac{\eps}{\delta^2} \cdot\sqrt{d} \cdot  \frac{\beta n t}{\sqrt{k}\cdot\normo{v}}\,.
\]

Consider the term $\bm W^\top \bm W$. By the Chi-squared tail bound (\cref{fact:chi-squared-tail-bounds}), 
with probability at least $1-\exp\paren{-\tau/2}$, 
\[
\Norm{\Paren{\bm W^\top \bm Ws(j)}_{\cL\setminus\supp(s(j))}} \le 2\sqrt{tn\cdot \Paren{\Card{\cL} + \tau}}\,.
\]
Since $\Card{\cL} \le k$, with probability at least $1-\exp\paren{-k}$,
\[
\Norm{\Paren{\bm W^\top \bm Ws(j)}_{\cL\setminus\supp(s(j))}} \le 10\sqrt{ktn} 
\le  \delta\Norm{\beta \norm{\bm u}^2 \iprod{v,s(j)} \cdot v}\,,
\]
where we used $\beta \gtrsim \frac{k}{\delta^3\sqrt{nt}}$.
%Moreover, with probability at least $1-\exp\paren{-d}$,
%\[
%\Norm{\Paren{\bm W^\top \bm Ws(j)}_{[d]\setminus\supp(s(j))}} \le 10\sqrt{dtn/\delta} \le  \frac{\delta \beta t n\sqrt{d}}{10\normo{v}}\,,
%\]
%where we used $\normo{v}\le \sqrt{k}$.

Consider the term $\sqrt{\beta}  \bm W^\top \bm u v^\top$. Since $\bm W$ and  $\bm u$ are independent, with probability at least $1-\exp(-k/2)$,
\[
\Norm{\Paren{\sqrt{\beta}  \bm W^\top \bm u v^\top s(j)}_{\cL}} \le 2\sqrt{\beta}\norm{\bm u} \sqrt{k} \cdot \iprod{v,s(j)} 
\le \delta\Norm{\beta \norm{\bm u}^2 \iprod{v,s(j)} \cdot v}\,,
\]
where we used $\beta n \gtrsim \frac{k\sqrt{n}}{\delta^2\sqrt{t}}$ and the fact that $n \ge t$.
%Moreover, with probability at least $1-\exp(-d/2)$,
%\[
%\Norm{\sqrt{\beta}  \bm W^\top \bm u v^\top s(j)} 
%\le \frac{\delta\sqrt{d}}{\sqrt{k}}\Norm{\beta \norm{\bm u}^2 \iprod{v,s(j)} \cdot v}\,.
%\]

Consider the term $\sqrt{\beta}  \bm W^\top \bm u v^\top$. With probability at least $1-\exp(-0.1\sqrt{nt})$,
\[
 \Norm{{\sqrt{\beta} v \bm u^\top \bm Ws(j)}} 
 \le 2\sqrt{\beta} \cdot \sqrt{t} \cdot \norm{\bm u}\cdot \paren{nt}^{1/4} \le  
 \delta\Norm{\beta \norm{\bm u}^2 \iprod{v,s(j)} \cdot v}\,,
\]
where we used the fact that the distribution of $\bm u ^\top\bm Ws(j)$ given $\norm{\bm u}$ is $\sim N(0, t\cdot \norm{u}^2)$.

\end{proof}

For $s\in \cS_t$ let  ${z}_s\Paren{r}$  be the $n$-dimensional (random) vector defined as 
\[
{z}_{si}\Paren{r} =  
\begin{cases}
\ind{\Paren{ \tilde{Y}^\top \tilde{Y} s}_i \ge r\cdot t\cdot n} \quad \text{if } s_i = 0\\
1 \qquad\qquad\qquad\;\; \text{otherwise}
\end{cases}
\]

\begin{lemma}\label{lem:perturbations-noise-decomposition}
	Suppose that 
	$k\gtrsim t\ln d$ and
	\[
	\beta \gtrsim \frac{k}{\delta^2\sqrt{tn}}\sqrt{\ln\Paren{2 + \frac{td}{k^2}\cdot \Paren{1 + \frac{d}{n} }}}\,.
	\]
	Let $r$ be such that
	\[
	\frac{\beta\delta}{1000\normo{v}^2}\le  r \le \frac{\beta\delta}{500\normo{v}^2}\,.
	\]
	Let $s\in \cS_t$ and
	\[
	N(s) = \Paren{\tilde{Y}^\top\tilde{Y} - n\Id - \beta \norm{\bm u}^2 vv^\top} \circ \Paren{{z}_{s}\Paren{r} {z}_{s}^\top\Paren{r}}\,.
	\]
Let 
\[
\cP_k = \Set{X \in \R^{d\times d} \;\suchthat\; X\succeq 0\,, \Tr{X} = 1 \,, \Normo{X}\le k}\,.
\]

	Then with probability $1-10\cdot d^{-10}$, for all $X \in \cP_k$,
	\[
	\Abs{\Iprod{X, N}} \le \delta\beta n\,.
	\]
\end{lemma}
\begin{proof}
			To simplify the notation we write $z$  instead of ${z}_{s}\Paren{r}$ and $N$ instead of $N(s)$.
	Let 
	\[
	N_E = \Paren{\tilde{Y}^\top \tilde{Y} - \Paren{\tilde{Y} - E}^\top  \Paren{\tilde{Y} - E}} \circ \Paren{zz^\top}\,.
	\]
	Note that
	\[
	\Abs{\Iprod{X, N}} \le \Norm{N - N_E} + \Abs{\Iprod{X, N_E}} \,.
	\]
	By \cref{lem:perturbation-inner-product},
	\[
	\Abs{\Iprod{X, N_E}} \le b^2 k + 2b\sqrt{k \Norm{\Paren{\tilde{Y} - E}^\top  \Paren{\tilde{Y} - E} \circ \Paren{zz^\top}}}\,.
	\]
	The first term can be bounded as follows: 
	\[
	b^2 k \le \eps \beta^2 n/k \le \delta \beta \sqrt{n}/100 \le \delta \beta n / 100\,.
	\]
	Note that with probability at least $1-\exp\paren{-n}$,  
	\[
	 \Norm{\Paren{\tilde{Y} - E}^\top  \Paren{\tilde{Y} - E}\circ \Paren{zz^\top}} \le 10\Paren{\beta n + n}\,.
	 \]
	Hence for the second term,
	\[
	b\sqrt{k\Paren{\beta n + n}} \le \eps \cdot \frac{2\beta}{1 + \sqrt{\beta}} n \Paren{1 + \sqrt{\beta}} \le \eps \beta n \le \delta \beta n / 100\,.
	\]
	Therefore,
	\[
	\Abs{\Iprod{X, N_E}}  \le \delta \beta n /10\,.
	\]
	
	By \cref{fact:markov-thresholding} and the bound on 
	$\Norm{\Paren{E^\top\bm{Y} + \bm{Y}^\top E}s'}$ from  \cref{lem:perturbations-column-norm-of-noise}, 
	at most $\tilde{\eps}^2 d= \Paren{10^4\eps/\delta}^2 d$ entries of $\Paren{E^\top\tilde{Y} + \tilde{Y}^\top E}s'$ 
	are larger (in absolute value) than $rtn/2$.
By \cref{lem:wishart-number-of-nonzero-entries}, with probability at least $1-d^{-10}$, number of entries $i$ such that
\[
\Abs{\Paren{\bm W^T Y s'}_i} \ge t\cdot r \cdot n/2
\]
is bounded by $pd + 10\sqrt{pd \ln d}$, where $p = \exp\Paren{-tnr^2/8}$.
Hence number of nonzero entries of $z$ is at most $10\tilde{\eps}^2d + pd + 10\sqrt{pd \ln d} + 2k \le 10\tilde{\eps}^2d + 10pd + 10k$ 
with probability at least $1-d^{-10}$.
	
	Therefore, by \cref{lem:spectral-bound} and \cref{lem:spectral-bound-technical}, with probability at least $1-2d^{-10}$,
	\[
		\Norm{N - N_E}  \le \delta\beta n/10 + 
		\sqrt{\Paren{n+\beta n}\cdot \tilde{\eps}^2 d  \ln\Paren{1/\tilde{\eps}}} + \tilde{\eps}^2 d  \ln\Paren{1/\tilde{\eps}}\,.
	\]
	Since the second and the third terms are bounded by $\delta \beta n / 10$, we get the desired bound.
\end{proof}

\begin{lemma}\label{lem:correlation-with-argmax-perturbations}
	Let $0 < \delta < 0.1$ and  let $\frac{\beta\delta^3}{1000\normo{v}^2}\le  r \le \frac{\beta\delta^3}{500\normo{v}^2}$.
	
	For $s\in \cS_t$ let $\hat{\bm v}(s)$ be the top
	eigenvector of $X\in \cP_k$ that maximizes $\Iprod{X, \tilde{Y} \circ \Paren{\tilde{z}_{s}\tilde{z}_{s}^\top}}$.
	
	Suppose that $n \gtrsim k + t \ln^2 d$, 
	$k\gtrsim t\ln d$ and
	\[
	\beta \gtrsim \frac{k}{\delta^6\sqrt{tn}}\sqrt{\ln\Paren{2 + \frac{td}{k^2}\cdot \Paren{1 + \frac{d}{n} }}}\,.
	\]
	
	Then there exists $s'\in \cS_t$ such that with probability $1-20d^{-10}$,
	\[
	\abs{\iprod{\hat{\bm v}(s'), v}} \ge 1 - \delta/10\,.
	\]
\end{lemma}
\begin{proof}
	Let $s'$ be as in \cref{lem:perturbations-column-norm-of-noise}
and let $\tilde{\bm v} = v \circ  z_{s'}(r)$.
Then
\[
\Paren{\tilde{Y}^\top \tilde{Y} - n\Id}\circ\Paren{ z_{s'}\Paren{r}  z^\top_{s'}\Paren{r}}= \beta \Norm{\bm u}^2\tilde{\bm v} \tilde{\bm v}^\top 
+ { N}(s)\,,
\]
where ${N}(s)$ is the same as in \cref{lem:perturbations-noise-decomposition}. 
By \cref{lem:perturbations-noise-decomposition},  with probability $1-10\cdot d^{-10}$, for all $X \in \cP_k$,
\[
\Abs{\Iprod{X, N}} \le \delta\beta n\,.
\]
Consider 
\[
\hat{{\bm X}}(s^*) \in 
\argmax_{X\in \cP} \Iprod{X, \tilde{Y} \circ \Paren{\tilde{z}_{s^*}\tilde{z}_{s^*}^\top}}\,.
\]
By \cref{lem:maximizer-lemma},
\[
\Iprod{\hat{{\bm X}}(s^*), vv^\top \circ \Paren{\tilde{z}_{s^*}\tilde{z}_{s^*}^\top} }
\ge 1 - 6\tilde{\delta}\,.
\]
Hence by \cref{fact:three-close-vectors},
\[
\Iprod{\hat{{\bm X}}(s^*), vv^\top}
\ge 1 - 24\tilde{\delta}\,.
\]

By \cref{lem:linear-algebra-correlation-eigenverctor-large-quadratic-form},
$\abs{\iprod{\hat{\bm v}(s'), \tilde{\bm v}}} \ge 1-100\tilde{\delta}$.
By \cref{lem:perturbations-column-norm-of-noise} and \cref{lem:small-perturbation-thresholding}, with probability $1-d^{-10}$, 
\[
\iprod{\tilde{\bm v}, v} \ge 1 - 100\tilde{\delta}\,.
\]
Using \cref{fact:three-close-vectors} we get the desired bound.
\end{proof}

\begin{proof}[Proof of \cref{thm:wishart-perturbations-technical}] 
	Since $\sqrt{n/2} \le \Norm{\bm u} \le\sqrt{2n}$ with probability at least $1-\exp\paren{-n/10}$, 
	in this proof we assume that this bound on $\Norm{\bm u}$ holds.
	Let $\tilde{\delta} = \delta / 1000$.
	Note that we can apply \cref{lem:perturbations-noise-decomposition} if 
	$\frac{\beta\tilde{\delta}^3 }{1000\normo{v}^2}\le  r \le \frac{\beta\tilde{\delta}^3 }{500\normo{v}^2}$.
	Since we do not know $\beta$ and $\normo{v}$, we can create a list of candidates for $r$ of size at most $2\ln n$ 
	(starting from $r = 1/n$ and finishing at $r = 1$).
	
	For all $s\in \cS_t$ and for all candidates for $r$ we compute $\hat{\bm v}(s)$ as in \cref{lem:correlation-with-argmax-perturbations}. By \cref{lem:correlation-with-argmax-perturbations},  we get a list of vectors $\bm L$ of size $2\ln (n)\cdot \card{\cS_t}$ such that with probability $1-20\ln(n)d^{-10} \ge 1 - d^{-9}$  there exists $\bm v^* \in \bm L$ such that $\abs{\iprod{\bm v^*, v}} \ge 1 - \tilde{\delta} / 10$. 
	
	By  \cref{fact:k-sparse-norm-gaussian}, \cref{lem:spectral-norm-cross}, \cref{lem:perturbation-inner-product},
	$k'$-sparse norm of $\bm Y^\top \bm Y - n\Id - \beta\norm{\bm u}^2 vv^\top$ is bounded
	by
	\[
	10\sqrt{nk'\ln\Paren{ed/k'}}+ 10\sqrt{\beta n k'\ln\Paren{ed/k'}} +  b^2 k' + 10b\sqrt{k'\Paren{\beta n + n}}
	\]
	with probability at least $1-3d^{-9}$. Let us show that if  
	\[
	k' \lesssim \min\Set{\delta \beta n/b^2, \frac{\delta^2\beta^2 n}{\Paren{1+\beta}\ln d}, \frac{\delta^2\beta^2 n}{\Paren{1+\beta}b^2\ln d} }\,,
	\]
	then $k'$ the $k'$-sparse norm is bounded by $\delta \beta n$.  Indeed, if $\beta \ge 1$, then
	\[
	\frac{\delta^2\beta^2 n}{\Paren{1+\beta}\ln d} \ge \frac{\delta^2\beta n}{2\ln d} \gtrsim k\sqrt{\frac{n}{t \ln^2 d}} \gtrsim k/\delta^2\,,
	\]
	and if $\beta < 1$,
	\[
	\frac{\delta^2\beta^2 n}{\Paren{1+\beta}\ln d} \ge \frac{\delta^2\beta^2 n}{2\ln d} \gtrsim \frac{k^2}{\delta^2 t \ln d} \gtrsim k/\delta^2\,,
	\]
	and by our bound on $b$, 
	\[
	k/\delta^2 \lesssim \min\Set{\delta \beta n/b^2, \frac{\delta^2\beta^2 n}{\Paren{1+\beta}b^2\ln d} }\,.
	\]
	
	By \cref{lem:list-decoding}, for $k' = \lceil100k/\delta^2 \rceil$,
	we can compute a $k'$-sparse vector unit vector $\hat{\bm v}(k')$ 
	such that
	$\Abs{\Iprod{\hat{\bm v}(k'), v}}\ge 1 - \delta$.
\end{proof}
		\section{The Wigner Model}
\label{sec:Wigner}

\subsection{Classical Settings}

In classical settings the input is an $d\times d$ matrix $\bm Y = \lambda v v^\top + \bm W$, 
where $v\in \R^d$ is a $k$-sparse unit vector, $\bm W\sim N(0,1)^{d\times d}$.
The goal is to compute a unit vector $\hat{\bm v}$ such that $\abs{\iprod{\hat{\bm v}, v}} \ge 0.99$ 
with high probability (with respect to the randomness of $\bm W$).

In this section we prove the following theorem.
	\begin{theorem}\label{thm:wigner-classical-technical}
	Let $d,k,t\in \N$, $\lambda > 0$, $0 < \delta < 0.1$. 
	Let $\bm Y = \sqrt{\beta}\bm u v^\top + \bm W$, 
	where $v\in \R^d$ is a $k$-sparse unit vector, $\bm W\sim N(0,1)^{d\times d}$.
	
	Suppose that
	$k\gtrsim {\frac{t\ln d}{\delta^2}}$, and
	\[
	\lambda \gtrsim \frac{k}{\delta^2}\sqrt{\frac{\ln\Paren{2 + {td}/{k^2}}}{t}}\,.
	\]
	
	Then there exists an algorithm that, given $\bm Y$, $k$, $t$ and $\delta$, 
	in time $d^{O(t)}$ outputs a unit vector $\hat{\bm v}$ such that
	with probability $1-o(1)$ as $d\to \infty$,
	\[
	\abs{\iprod{\hat{\bm v}, v}} \ge 1 - \delta\,.
	\]
\end{theorem}

	As for Wishart model, we define vectors $\bm z_{s}(r)$ that we will use in the algorithm. 
	Recall the definition of $\cS_t$: for $t\in \N$ such that $1\le t \le k$ we denote by$\cS_t$  
	the set of all $d$-dimensional vectors with values in $\set{-1,0,1}$ 
	that have exactly $t$ nonzero coordinates.
	For $s\in \cS_t$ let  $\bm z_s$  be $d$-dimensional (random) vectors defined as 
	
\[
\bm z_{si}\Paren{r} =  
\begin{cases}
\ind{\Paren{\bm Y s}_i = \iprod{s, \bm Y_i} \ge r\cdot t} \quad \text{if } s_i = 0\\
1 \qquad\qquad\qquad\quad\;\; \text{otherwise}
\end{cases}
\]
	
	for some $r > 0$ (here $\bm Y_i$ denotes the $i$-th row of $\bm Y$).
	
	\begin{lemma}\label{lem:wigner-erasing-probability}
		If $i\notin \supp\paren{v}$, then for all $s\in  \cS_t$ and $r>0$,
		\[
		\Pr\Brac{\iprod{s, \bm Y_i}\ge r\cdot t} \le \exp\Paren{-t r^2 / 2}\,.
		\]
		\begin{proof}
			Since $v_i = 0$,
			\[
			\iprod{s, \bm Y_i} = \iprod{s, \bm W_i} \sim N(0, t)\,.
			\]
			The lemma follows from the tail bound for Gaussian distribution (\cref{fact:gaussian-tail-bound}).
		\end{proof}
	\end{lemma}

	\begin{lemma}\label{lem:wigner-signal-preserved}
		Let $0 < \delta < 0.1$, $r \le  \frac{\lambda \delta }{100k}$ 
	and suppose that $\lambda \gtrsim \frac{k}{\delta^2\sqrt{t}}$ and
	$k\gtrsim {\frac{\ln d}{\delta^2}}$. 
	
	Let $s^* \in \argmax_{s\in \cS_t}\iprod{s, v}$.
	Then, with probability $1-d^{-10}$
	\[
	\norm{v\circ \bm z_{s^*}}^2 = \iprod{v\circ \bm z_{s^*}, v} \ge 1 - \delta\,.
	\]
	\begin{proof}
				To simplify the notation we write $\bm z_{si}$  instead of $\bm z_{si}\Paren{r}$. 
		Let $\cL_\delta = \Set{i \in [d] \suchthat \abs{v_i} \ge \frac{\sqrt{\delta}}{10\sqrt{k}}}$.
				  For all $i\in \cL$, $\beta \norm{\bm u}^2  \Iprod{v,s^*} \abs{v_i}\ge 10rt$. Note that $\norm{v_\cL}^2 \ge 1 - \delta / 10$. 
				  
		As in the proof of \cref{lem:wishart-signal-preserved}, we can assume that 
		 \[
		 \iprod{s^*, v} = \normo{v_T} \ge \sqrt{\delta}t/\sqrt{k}\,.
		 \]
				
		 We need to bound the norm of $\bm Ws^* \sim N(0, t\cdot \Id)$ restricted to the entries of $\cL$. 
		 By \cref{fact:chi-squared-tail-bounds}, with probability $1-\exp(-k/10)$,
		 \[
		 \Norm{\Paren{\bm Ws^*}_\cL}\le 2\sqrt{kt} \le \frac{\delta}{100} \Norm{\lambda \Iprod{v, s^*} v}\,.
		 \]
		
		By \cref{lem:small-perturbation-thresholding}, 
		\[
		\norm{v_{\cL} \circ \bm z_{s^*}}^2 \ge \Paren{1-\delta/10} \norm{v_\cL}^2 \ge 1 - \delta\,.
		\]
	\end{proof}
\end{lemma}

\begin{lemma}\label{lem:wigner-spectral-bound}
	Let $p = \exp\Paren{-t r^2/2}$.
	Then, with probability $1-d^{-10}$,
	\[
	\max_{s\in \cS_t} \Norm{\bm W \circ \Paren{\bm z_s\Paren{r}\bm z^\top_s\Paren{r}} } \le 10\sqrt{\Paren{pd + k} \cdot \ln\Paren{\frac{d}{pd + k}}}\,.
	\]
\end{lemma}
\begin{proof}
Using the same argument as in the proof of \cref{lem:wishart-spectral-bound},
we get that the number of nonzero $\bm z_{si}(r)$ for every $s\in \cS_t$ 
is bounded by $2pd + 2\sqrt{pd t\ln \paren{n/t}} + 2k \le 4pd  + 4k$ with probability at least $1-\exp\Paren{-pd - t\ln \paren{d/t} }$.
	
	By \cref{fact:bound-covariance-gaussian}, 
	an $m\times m$ Gaussian matrix $\bm G$ satisfies
	\[
	\Norm{\bm G} \le 2\sqrt{m} + \sqrt{\tau}
	\]
	with probability $1-\exp(-\tau / 2)$ (for every $\tau > 0$).
	By union bound over all sets of size at most $4pd  + 4k$ 
	(corresponding to nonzero rows and columns of $\bm W \circ \Paren{\bm z_s\Paren{r}\bm z^\top_s\Paren{r}} $), 
	we get the desired bound.
\end{proof}
	
\begin{lemma}\label{lem:wigner-correlation-with-argmax}
Let $0 < \delta < 0.1$ and  $\frac{\delta\lambda}{200k} \le r \le \frac{\delta\lambda}{100k}$. 
For $s\in \cS_t$ let $\hat{\bm v}(s)$ be the top eigenvector of $\bm Y\circ \Paren{\bm z_s\Paren{r}\bm z^\top_s\Paren{r}}$. 

Suppose that
	$k\gtrsim {\frac{t\ln d}{\delta^2}}$, and
\[
\lambda \gtrsim \frac{k}{\delta^2\sqrt{t}}\sqrt{\ln\Paren{2 + \frac{td}{k^2}}}\,.
\]

Then there exists $s'\in \cS_t$ such that with probability $1-2d^{-10}$,
\[
\abs{\iprod{\hat{\bm v}(s'), v}} \ge 1 - 20\delta\,.
\]
\end{lemma}

\begin{proof}
	Let $s' \in \argmax_{s\in \cS_t}\iprod{s, v}$ 
and let $\tilde{\bm v} = v\circ \bm z_{s'}\Paren{r}$. Then
\[
\bm Y\circ \Paren{\bm z_s\Paren{r}\bm z^\top_s\Paren{r}} = \beta \norm{\bm u}^2\tilde{\bm v} \tilde{\bm v}^\top + 
\bm W \circ \Paren{\bm z_s\Paren{r}\bm z^\top_s\Paren{r}} \,,
\]
By \cref{lem:wigner-spectral-bound} and the same arument as in \cref{lem:spectral-bound-technical} with $n = d$, with probability at least $1 - d^{-10}$,
\[
\Norm{\bm W \circ \Paren{\bm z_s\Paren{r}\bm z^\top_s\Paren{r}} } \le \delta \lambda\,.
\]

By \cref{lem:maximizer-lemma},
$\abs{\iprod{\hat{\bm v}(u^*), \tilde{\bm v}}} \ge 1-3\delta$.
By \cref{lem:wigner-signal-preserved}, with probability $1-d^{-10}$, $\iprod{\tilde{\bm v}, v} \ge 1 - \delta$. 
Therefore, by \cref{fact:three-close-vectors},
\[
\Abs{\iprod{\hat{\bm v}(s'), v}} \ge 1 - 20\delta\,.
\]
\end{proof}

\begin{proof}[Proof of \cref{thm:wigner-classical-technical}] 
	For all $s\in \cS_t$ we compute $\hat{\bm v}(s)$ as in \cref{lem:wigner-correlation-with-argmax} 
	with $\frac{\delta'\lambda}{200k} \le r \le \frac{\delta'\lambda}{100k}$, where $\delta'=\delta / 1000$.
	Since we do not know $\lambda$, we can create a list of candidates for $r$ of size at most $2\ln d$ 
	(from $r = 1/d$ to $r = 1$). 
	
	By \cref{lem:wigner-correlation-with-argmax},  we get a list of vectors $\bm L$ of size $2\ln (d)\card{\cS_t}$ such that with probability $1-d^{-9}$ 
	there exists $\bm v^* \in \bm L$ such that $\abs{\iprod{\bm v^*, v}} \ge 1 - \delta / 10$. 
	By \cref{fact:k-sparse-norm-gaussian}, $k'$-sparse norm of $\bm W$ is bounded by 
   \[
   10\sqrt{k'\ln\Paren{ed/k'}}\,
   \] 
   with probability at least $1-d^{-10}$.
   Hence by \cref{lem:list-decoding} with $k' = \lceil 100k/\delta^2\rceil$, we get the desired estimator.
\end{proof}

\subsection{Adversarial Perturbations}

The following theorem is the restatement of \cref{thm:wigner-results}.
\begin{theorem}\label{thm:wigner-perturbations-technical}
	Let $d,k,t\in \N$, $\lambda > 0$, $0 < \delta < 0.1$. 
	Let $\tilde{Y} = \sqrt{\beta}\bm u v^\top + \bm W + E$, 
	where $v\in \R^d$ is a $k$-sparse unit vector, $\bm W\sim N(0,1)^{n\times d}$, and $E\in \R^{d\times d}$ is a matrix with entries
	\[
	\Normi{E} = \eps \lambda/k \lesssim \delta^3 \lambda/k\,.
	\]
	
	Suppose that
	$k\gtrsim {\frac{t\ln d}{\delta^2}}$  and
	\[
	\lambda \gtrsim \frac{k}{\delta^2}\sqrt{\frac{\ln\Paren{2 + {td}/{k^2}}}{t}}\,.
	\]
	
	Then there exists an algorithm that, given $\tilde{Y}$, $k$, $t$ and $\delta$, 
	in time $d^{O(t)}$ outputs a unit vector $\hat{\bm v}$ such that
	with probability $1-o(1)$ as $d\to \infty$,
	\[
	\abs{\iprod{\hat{\bm v}, v}} \ge 1 - \delta\,.
	\]
\end{theorem}

\begin{proof}
	Let $\tilde{\delta} = \delta / 1000$ and let $r = \frac{\tilde{\delta}}{100 k}$.
	For $s\in \cS_t$ let  $\tilde{z}_{s}$  be $n$-dimensional (random) vectors defined as
	\[
	\tilde{z}_{si}\Paren{r} =  
	\begin{cases}
	\ind{\iprod{s, \tilde{Y}_i} \ge r\cdot t} \quad \text{if } s_i = 0\\
	1 \qquad\qquad\qquad\quad\;\; \text{otherwise}
	\end{cases}
	\]
	Let $\bm Y = \tilde{Y} - E$, and $\bm z_{si}$ be the same as in the non-adversarial case defined for $\bm Y$.
	Note that 
	\[
	\bm z_{si}\Paren{r - \eps\cdot \lambda/k} \le \tilde{z}_{si}\Paren{r} \le \bm z_{si}\Paren{r + \eps\cdot \lambda/k}\,.
	\]
	Let $s^* \in \argmax_{s\in \cS_t}\iprod{s, v}$.
	By \cref{lem:wigner-signal-preserved}, 
	with probability $1-d^{-10}$,
	\[
	\norm{v\circ \tilde{z}_{s^*}}^2 = \iprod{v\circ \tilde{z}_{s^*}, v} \ge \bm z_{si}\Paren{r - \eps\cdot \lambda/k} \ge 1 - 2\tilde{\delta}\,.
	\]
	
	By the argument from the proof of  \cref{lem:wishart-spectral-bound},
	with probability $1-d^{-10}$,  number of nonzero entries of $\tilde{z}_{s}$ is at most $4pd+4k$
	and 
	\[
	\Norm{\bm W\circ\Paren{\tilde{z}_{s^*}\tilde{z}_{s^*}^\top}} \le 10\sqrt{\Paren{pd + k} \cdot \ln\Paren{\frac{d}{pd + k}}}\lesssim \tilde{\delta}^2 \lambda\,.
	\]
	
	Let $\cP_k = \Set{X \in \R^{d\times d} \;\suchthat\; X\succeq 0\,, \Tr{X} = 1 \,, \Normo{X}\le k}$.
	For all $X \in \cP_k$,
	\[
	\Abs{\Iprod{X, \bm W\circ\Paren{\tilde{z}_{s^*}\tilde{z}_{s^*}^\top}}} 
	\le \Norm{\bm W\circ\Paren{\tilde{z}_{s^*}\tilde{z}_{s^*}^\top}} \cdot \Norm{X} 
	\lesssim  \tilde{\delta} \lambda\,.
	\]
	and
	\[
	\Abs{\Iprod{X,  E\circ\Paren{\tilde{z}_{s^*}\tilde{z}_{s^*}^\top}}} 
	\le \Normi{E\circ\Paren{\tilde{z}_{s^*}\tilde{z}_{s^*}^\top}} \cdot\Normo{X} 
	\le \tilde{\delta}\lambda\,.
	\]
	Consider 
	\[
	\hat{{ X}}(s^*) \in 
	\argmax_{X\in \cP} \Iprod{X, \tilde{Y} \circ \Paren{\tilde{z}_{s^*}\tilde{z}_{s^*}^\top}}\,.
	\]
	By \cref{lem:maximizer-lemma},
	\[
	\Iprod{\hat{{ X}}(s^*), vv^\top \circ \Paren{\tilde{z}_{s^*}\tilde{z}_{s^*}^\top} }
	\ge 1 - 6\tilde{\delta}\,.
	\]
	Hence by \cref{fact:three-close-vectors},
	\[
	\Iprod{\hat{{ X}}(s^*), vv^\top}
	\ge 1 - 24\tilde{\delta}\,.
	\]
	
	Let $\hat{v}(s^*) $ be the top eigenvector of $\hat{ X}(s^*)$. 
	By \cref{lem:linear-algebra-correlation-eigenverctor-large-quadratic-form},
	\[
	\Abs{\Iprod{{\hat{v}}(s^*) , v}} \ge 1 - 100\tilde{\delta}\,.
	\]
	By \cref{fact:k-sparse-norm-gaussian}, $k'$-sparse norm of $\bm W$ is bounded by 
	\[
	10\sqrt{k'\ln\Paren{ed/k'}}
	\] 
	with probability at least $1-d^{-10}$.
	And $k'$-sparse norm of $E$ is bounded by
	\[
	k'\Normi{E} \le \eps \lambda \frac{k'}{k}\,.
	\]
	Hence by \cref{lem:list-decoding} with $k' = \lceil 100k/\tilde{\delta}^2\rceil$ and the list of $\bm X(s)$ for all $s\in \cS_t$, 
	we get the desired estimator.
\end{proof}

		\section{Heavy-tailed Symmetric Noise}
\label{sec:Symmetric}

The following theorem is a restatement of \cref{thm:wigner-symmetric-results}.
\begin{theorem}\label{thm:wigner-symmetric-technical}
	Let $k, d, t\in\N$, $\lambda > 0$, $A \ge 1$, $0 < \alpha < 1$, $0 < \delta < 0.1$. 	
	Let 
	\[
	\bm Y = \lambda vv^\top + \bm N\,,
	\] 
	where $v\in \R^d$ is a $k$-sparse unit vector such that $\normi{v}\le A/\sqrt{k}$
	and $\bm N\sim N(0,1)^{d\times d}$ is a random matrix with independent
	(but not necessarily identically distributed) symmetric about zero entries such that 
	\[
	\Pr\Brac{\Abs{\bm N_{ij}} \le 1} \ge \alpha\,.
	\]
	
	Suppose that $k \gtrsim \frac{t \ln d}{\delta^4 A^4\alpha^2}$, 
	\[\
	t \gtrsim  \frac{\ln\Paren{2+ td/k^2}}{\alpha^2 A^4 \delta^{6}}\,,
	\]
	and
	\[
	\lambda \ge k\,.
	\]
	
	Then there exists an algorithm that, given $\bm Y$, $k$, $t$ and $\lambda$, 
	in time $ d^{O(t)}$ finds a unit vector $\hat{\bm v}$ such that
	with probability $1-o(1)$ as $d\to \infty$,
	\[
	\abs{\iprod{\hat{\bm v}, v}} \ge 1-O\Paren{\delta}\,.
	\]
\end{theorem}

Before proving this theorem, we state here a theorem from \cite{dNNS22}.
\begin{theorem}[\cite{dNNS22}]\label{thm:meta-theorem}
	Let $\delta, \alpha \in (0,1)$ and $\zeta \ge 0$.
	Let $\pdset \subseteq \R^m$ be a compact convex set. Let $b,r,\gamma\in\R$ be such that 
	$$\max_{X\in \pdset}\normi{X}\leq b\,,$$ 
	$$\max_{X\in \pdset}\norm{X}_2\leq r\,,$$ and $$\E_{\bm W\sim N(0,\Id)}\Brac{\sup_{X\in\pdset} \iprod{X, \bm W}}\leq \gamma\,.$$
	Consider
	\begin{align*}
	\bm Y = X^*+ \bm N\,,
	\end{align*}
	where $X^*\in \tilde{\Omega}$ and $\bm N$ is a random $m$-dimensional vector with independent (but not necessarily identically distributed) symmetric about zero entries satisfying $\Pr \brac{\Abs{\bm N_i}\leq \zeta}\geq \alpha$. 
	
	Then the minimizer $\hat{X} = \argmin_{X\in \pdset}{F_h(\bm Y-X)}$ of the Huber loss with parameter $h \ge 2b + \zeta$ satisfies
	\begin{align*}
	\Normt{\hat{X}-{X}^*}\leq 
	O\Paren{\sqrt{{\frac{h}{\alpha}}\Paren{\gamma + r\sqrt{\log(1/\delta)}}}} 
	\end{align*}
	with probability at least $1-\delta$ over the randomness of $\bm N$.
\end{theorem}

Using this result, we will prove \cref{thm:wigner-symmetric-technical}

\begin{proof}[Proof of \cref{thm:wigner-symmetric-technical}]
	Consider 
	\[
	\tau_{h}\Paren{x} :=
	\begin{cases}
	h\qquad\,\;\;  \text{if $x > h$}\\
	x \qquad\,\;\; \text{if $\Abs{x} \le h$}\\
	-h\qquad \text{if $x < -h$}
	\end{cases} 
	\]
	where $h = 3\lambda \normi{v}^2 \le 3\frac{\lambda}{k}A^2$. 
	Let $\bm T = \tau_h(\bm Y)$ (i.e. the matrix obtained from $Y$ by applying $\tau_h$ to each entry).
	
	For $s\in \cS_t$ let  $\bm z_s$  be $n$-dimensional (random) vectors defined as 
	\[
	\bm z_{si}\Paren{r} =  
	\begin{cases}
	\ind{\Paren{\bm T s}_i = \iprod{s, \bm T_i} \ge r\cdot t} \quad \text{if } s_i = 0\\
	1 \qquad\qquad\qquad\quad \text{otherwise}
	\end{cases}
	\]
	Let $\frac{\lambda \delta^2 \alpha}{10k} \le r \le  \frac{\lambda \delta^2 \alpha}{5k} $.
	
	Let $\cL= \Set{i\in [d] \suchthat \abs{v_i} \ge \frac{{\delta}}{\sqrt{k}}}$. Note that $\Norm{v_\cL}^2 \ge 1 -\delta^2$ and hence 
	$\card{\cL} \ge \frac{k}{2A^2}$.
	
	Let $s^* \in \argmax_{s\in \cS_t}\iprod{s, v}$. By the same argument as in the proof of \cref{lem:wishart-signal-preserved},
	\[
	\iprod{s^*, v} = \normo{v_T} \ge {\delta}t/\sqrt{k}\,.
	\]
	
	By Chernoff bound, for every $i\in \cL$ with probability at least $1-\exp(-t\alpha/2)$,
	for at least $\alpha t / 10$ of $j \in\supp(s^*(j))$, $\Abs{\bm N_{ij}} \le 1$.
	Let $\cC_{i}$ be the set of such entries.
	Then for $j\in \cC_i$,
	\[
	\tau_h\Paren{\bm Y_{ij}} = \lambda v_i v_j + \Abs{\bm N_{ij}}\,.
	\]
	By Hoeffding's inequality, 
	\[
	\sum_{j\in \cC_i} \tau_h\Paren{\bm Y_{ij}}  \ge 
	\lambda\frac{{\delta}\Card{\cC_i}}{\sqrt{k}}v_i - \sqrt{\Card{\cC_i}q} 
	\]
	with probability at least $1-\exp\paren{-q/2}$
	Note that for all $j \in \supp(j)$,
	\[
	\E\tau_h\Paren{\bm Y_{ij}} \ge 0\,.
	\]
	By Hoeffding's inequality,
	\[
	\sum_{j\in\supp(s^*(j))\setminus \cC_i}\tau_h\Paren{\bm Y_{ij}}  \ge - h\sqrt{tq} 
	\]
	with probability at least $1-\exp\paren{-q/2}$.
	Hence for $i\in \cL$, with probability at least $1-\delta^2/A^2$.
	\[
	\iprod{s^*, \bm T_i} \ge \frac{\lambda \delta^2 \alpha t}{2k} - 10h\sqrt{t \log(A/\delta)}
	\ge  \frac{\lambda \delta^2 \alpha t}{4k} \ge r\cdot t\,.
	\]
	Let $\mu = \delta^2/A^2$.
	By Chernoff bound, with probability at least $1 - \exp\Paren{-\mu k/10}$, for
	at most $2\mu k$ entries $i\in \cL$, $\bm z_{s^*i} = 0$. 
	Hence with probability at least $1 - \exp\Paren{-\mu \Card{\cL}/100}$,
	\[
	\norm{v_{\cL}\circ \bm z_{s^*}}^2 \ge 1 - 2A^2\mu \ge 1-\delta^2\,.
	\]
	
	By Hoeffding's inequality, for all $i\notin \supp(v)$,
	\[
	\Abs{\iprod{s^*, \bm T_i}} \le h\sqrt{tq}
	\]
	with probability at least $1 - \exp\paren{-q/2}$. Hence
	\[
	\Abs{\iprod{s^*, \bm T_i}} < r\cdot t
	\]
	with probability at least $1-\exp\Paren{-\frac{r^2 t}{10h^2}} = 1-p$. 
	
	By the same argument as in \cref{lem:wishart-spectral-bound}, number of nonzero $\bm z_{s^*i}(r)$ 
	is at most $4pd+4k$ with probability at least $1-\exp\Paren{-pd - t\ln \paren{d/t} }$. 
	For $s\in \cS_t$, denote by $\bm Z(s)$ the set of $i\in[d]$ such that $\bm z_{si}(r) = 1$. 
	
	For  $Q \subset [d]$ let 
	\[
	\cP_Q = \Set{X \in \R^{ Q\times  Q} \;\suchthat\; X\succeq 0\,, \Tr{X} \le \lambda \,, \Normo{X}\le \lambda k}\,,
	\]
	Note that
	\[
	\gamma(Q) := \E_{\bm G \sim N(0,1)^{ Q\times  Q} } \Brac{\sup_{X\in \bm \cP_Q } \Iprod{X, \bm G} } \le 
	\lambda\cdot \E_{\bm G \sim N(0,1)^{ Q\times  Q}} \Norm{\bm G} \le 10\lambda\sqrt{\Card{Q}}\,.
	\]
	
	Hence by \cref{thm:meta-theorem}
	and a union bound over all sets $Q$ of size at most $4pd + 4k$, we get with probability at least $1-d^{-10}$, for all $s\in \cS_t$,
	\[
	\Normf{\hat{\bm{X}}(s) - \lambda \tilde{\bm v}(s)\tilde{\bm v}(s)^\top}^2
	\le  O\Paren{{\frac{h}{\alpha} \cdot \lambda{\sqrt{\Paren{pd + k} \cdot \ln\Paren{\frac{d}{pd + k}}}}}} 
	\le O\Paren{{\frac{A^2}{\alpha} \cdot \frac{\lambda^2}{k} \cdot {\sqrt{\Paren{pd + k} \cdot \ln\Paren{\frac{d}{pd + k}}}}}}\,,
	\]
	where $\hat{\bm{X}}(s)$ is the minimizer of the Huber loss with parameter $h$ over $ \cP_{\bm Z(s)}$ and $\tilde{\bm v} = v \circ \bm z_{s}(r)$.
	
	Note that
	\[
	\sqrt{\Paren{pd + k} \cdot \ln\Paren{\frac{d}{pd + k}}} \lesssim \sqrt{pd \ln(1/p)} + \sqrt{k\ln d}\,.
	\]
	The second term can be bounded as follows
	\[
	\sqrt{k\ln d} \lesssim \delta^2 k \frac{\alpha}{A^2}\,.
	\]
	
	Note that
	\[
	r^2 t /h^2 \ge \frac{\delta^4 \alpha^2}{A^4} t \gtrsim 
	\ln\Paren{2 + td/k^2}\,.
	\]
	Hence the first term can be bounded as follows
	\[
	\sqrt{pd \ln(1/p)} \lesssim \frac{k}{\sqrt{d}}  \cdot \frac{\alpha\delta^2}{A^2} \cdot \sqrt{d} \le \delta^2 k \frac{\alpha}{A^2}\,.
	\]
	
	Hence for all $s\in \cS_t$,
	\[
	\Normf{\hat{\bm{X}}(s) - \lambda \tilde{\bm v}(s)\tilde{\bm v}(s)^\top}^2 \le \delta^2 \lambda^2\,.
	\]
	
	Since $\norm{v\circ \bm z_{s^*}}^2 \ge 1-2\delta^2$, 
	\[
	\Normf{\hat{\bm{X}}(s^*)} \ge 1 - 10\delta\lambda\,.
	\]
	
	Consider some $s'$ such that $\Normf{\hat{\bm{X}}(s')} \ge 1 - 10\delta\lambda$. 
	For such $s'$, 
	\[
	\iprod{v\circ \bm z_{s'}\Paren{r}, v}  = \Norm{v\circ \bm z_{s'}\Paren{r}}^2 \ge 1 - 100\delta\,.
	\]
	Moreover, since $\Normf{\hat{\bm{X}}(s') - \lambda \tilde{\bm v}(s')\tilde{\bm v}(s')^\top} \le \delta \lambda$, the top eigenvector $\hat{\bm v}(s')$ of $\hat{\bm{X}}(s')$ satisfies
	\[
	\abs{\iprod{\hat{\bm v}(s'), v}} \ge 1-O\Paren{\delta}\,.
	\]
	
\end{proof}
	%\input{content/meta-theorem}
	%\input{content/new-applications}
	%\input{content/lowerbounds}

	% BIBLIOGRAPHY
	\newpage
	
	% assumes hyperref
	\phantomsection
	\addcontentsline{toc}{section}{References}
	\bibliographystyle{amsalpha}
	\bibliography{bib/custom,bib/dblp,bib/mathreview,bib/scholar,bib/oldOblivious}
	\appendix
	
	% APPENDIX
	%\section{Facts about Huber loss}
%\label{sec:Huber-loss}

\section{Properties of sparse vectors}\label{section:sparse-vectors}

This section contain tools used throughout the rest of the paper.

%
%\begin{lemma}\label{lem:normFTensorDiff}\Gnote{This lemma is perhaps not needed}
%	Let $x,v\in \Set{-\frac{1}{\sqrt{n}}, +\frac{1}{\sqrt{n}}}^n$ and let $\Delta:= \tensorpower{x}{3}-\tensorpower{x}{3}$. Then $\Normt{\Delta}$ is related to the distance between $x$ and $v$ as follows:
%	\begin{equation*}
%		\frac{1}{\sqrt{2}}\norm{v-x}\leq \normt{\Delta}\leq 3\norm{v-x}\,.
%	\end{equation*}
%\end{lemma}
%\begin{proof}
%	For the upper bound, we have
%	\begin{align*}
%	\normt{\Delta}&=\Normt{v^{\otimes 3} - x^{\otimes 3}}\\
%		&\leq \Normt{v^{\otimes 3} - x\otimes v^{\otimes 2}}+\Normt{x\otimes v^{\otimes 2}-x^{\otimes 2}\otimes v}+\Normt{x^{\otimes 2}\otimes v - x^{\otimes 3}}\\
%		&= \norm{v-x}\cdot\norm{v}^2+\norm{x}\cdot\norm{v-x}\cdot\norm{v}+\norm{x}^2\cdot\norm{v-x}\\
%		&= 3\norm{v-x}\,,
%	\end{align*}
%	where in the last equality we used the fact that $v$ and $x$ are unit vectors. For the lower bound, we have
%	\begin{align*}
%		\normt{\Delta}^2&=\Normt{v^{\otimes 3}}^2 + \Normt{x^{\otimes 3}}^2 - 2\iprod{v^{\otimes 3},x^{\otimes 3}}=2 - 2\iprod{v,x}^3\\
%		&= (2-2\iprod{v,x})\cdot\left(1+\iprod{v,x}+\iprod{v,x}^2\right)\\
%		&\stackrel{\text{(a)}}\geq \left(\norm{v}^2 + \norm{x}^2-2\iprod{v,x}\right)\cdot\frac{1}{2}\\
%		&= \frac{1}{2}\cdot \norm{v-x}^2\,,
%	\end{align*}
%	where (a) follows from the fact that $1+t+t^2\geq \frac{1}{2}$ for every $-1\leq t\leq 1$, and the fact that $v$ and $x$ are unit vectors.
%\end{proof}
  
\begin{fact}\label{fact:markov-thresholding}
	Let $r, \delta >0$ and
	let $x\in \R^m$ such that $\norm{x}\le R$. Let $\cS = \Set{i\in [m] \;\suchthat\; \Abs{x_i} \ge \delta}$. Then
	\[
	\Card{\cS} \le R^2/\delta^2\,.
	\]
\end{fact}
\begin{proof}
	\[
	\delta^2 \cdot \Card{\cS} \le \Norm{x_{\cS}}^2 \le \Norm{x}^2 \le r^2\,.
	\]
\end{proof}

\begin{lemma}\label{lem:small-perturbation-thresholding}
	Let $\delta, \delta' \in (0,1)$.
	Let $x, y \in \R^m$ such that $\norm{y}\le \delta \norm{x}$. Let $\cS = \Set{i\in [m] \;\suchthat\; \Abs{y_i} \ge \delta' \Abs{x_i}}$. Then
	\[
	\norm{v_{\cS}} \ge \Paren{1-\delta/\delta'} \norm{v}\,.
	\]
\end{lemma}
\begin{proof}
	Consider the vector $y'$ such that $y'_i = -x_i$ for all $i\in \cS$ and $y'_i = 0$ for all $i\notin \cS$. It follows that
	\[
	\norm{y'} \le \frac{1}{\delta'}\norm{y} \le \frac{\delta}{\delta'}\norm{x}\,.
	\]
	Hence 
	\[
	\norm{x_{\cS}} = \norm{x + y'} \ge \norm{x}_2 - \norm{y'}\ge \Paren{1-\delta/\delta'} \norm{x}\,.
	\]
\end{proof}         

\begin{lemma}\label{lem:close-to-sparse-thresholding}
Let $v \in \R^d$ be a $k$-sparse unit vector, and suppose that for some unit vector $v' \in \R^d$, 
$\Abs{\Iprod{v,v'}} \ge 1-\delta$. For $k' \ge k$, let $\cK'$ be the set of $k'$ largest (in absolute value) entries of $v'$.
Then
\[
\Abs{\Iprod{v,v'_{\cK'}}} \ge 1-\delta - \sqrt{k/k'}\,.
\]
\end{lemma}
\begin{proof}
Since $\norm{v}=1$, $\Normi{v'_{\cK'} - v'} \le 1/k'$. Hence
\[
\Abs{\Iprod{v, v'_{\cK'} - v'}}  \le \Normi{v'_{\cK'} - v'}\cdot \Normo{v} \le \sqrt{k/k'}\,.
\]
Therefore,
\[
\Abs{\Iprod{v,v'_{\cK'}}}  \ge \Abs{\Iprod{v,v'}}  - \Abs{\Iprod{v, v'_{\cK'} - v'}} \ge 1 - \delta - \sqrt{k/k'}\,.
\]
\end{proof}

\begin{lemma}\label{lem:list-decoding}
	Let $\lambda, \kappa, \delta > 0$ and 
	let $v \in \R^d$ be a $k$-sparse unit vector. 
	Let $L \subset \R^d$ be a finite set of unit vectors such that
	\[
	\max_{x\in L} \Abs{\Iprod{v,x}} \ge 1-\delta\,.
	\]
	Let $k' \ge 2\lceil k/\delta^2 \rceil$ and let $N \in \R^{d\times d}$ be a matrix such that for every $k'$-sparse unit vector $u\in \R^d$
	\[
	u^\top N u \le \kappa\,.
	\]
	Then, there exists an algorithm running in time $O\Paren{d^2 \cdot L}$ that, given $Y = \lambda vv^\top + N, L, k$ and $\delta$ as input, finds a unit vector $\hat{v}$ such that
	\[
	\Abs{\Iprod{v,\hat{v}}} \ge 1 - 4\delta - \frac{2\kappa}{\lambda}\,.
	\]
\end{lemma}
\begin{proof}
	For each $x \in L$ we can compute a $k'$-sparse vector $s(x)$ that coincides with $x$ on the top $k'$ largest (in absolute value) entries.
	Let $x^* \in \argmax_{x\in L} \Abs{\Iprod{v,s(x)}}$. By \cref{lem:close-to-sparse-thresholding},
	\[
	\Abs{\Iprod{v,s\Paren{x^*}}} \ge 1 - 2\delta\,.
	\]
	Hence 
	\[
	s\Paren{x^*}^\top \Paren{\lambda vv^\top + N} s\Paren{x^*} \ge \Paren{1-2\delta}^2\lambda - \kappa \ge  \Paren{1-4\delta}\lambda - \kappa\,.
	\]
	Let
	\[
	\tilde{v} \in \argmax_{s(x)\,, x\in L} s\Paren{x}^\top \Paren{\lambda vv^\top + N} s\Paren{x}\,.
	\]
	Then 
	\[
	\tilde{v}^\top \Paren{\lambda vv^\top + N} \tilde{v} \ge s\Paren{x^*}^\top \Paren{\lambda vv^\top + N} s\Paren{x^*} 
	\ge \Paren{1-4\delta}\lambda - \kappa\,.
	\]
	Since $\tilde{v}^\top N \tilde{v}  \le \kappa$,
	we get 
	\[
	\Abs{\Iprod{v,\tilde{v}}} \ge \Abs{\Iprod{v,\tilde{v}}}^2 \ge 1-4\delta - 2\kappa/\lambda\,.
	\]
	Hence $\hat{v} = \frac{1}{\norm{\tilde{v}}} \tilde{v}$ is the desired estimator.
\end{proof}

%
%\begin{fact}\label{fact:expectation_nonnegative_variable}
%	Let $\bm \xi$ be a nonnegative random variable. Then $\B \bm\xi = \int_{0}^\infty \Pr\Brac{\bm\xi > \tau} d\tau$.
%\end{fact}
%
%\begin{lemma}\label{lemma:expectation_from_tails}
%	Let $\bm\eta$ be a random variable such that for some $a \in \R$ and for all $\tau > 0$, $\bm \eta \le a + \tau$ with probability at least $1-f(\tau)$ for some nonnegative $f\in \cL_1\Paren{(0, \infty)}$. Then
%	\[
%	\B \bm\eta \le a + \int_0^\infty f(\tau) d\tau\,.
%	\]
%\end{lemma}
%\begin{proof}
%	Denote $\bm \xi =  \ind{\bm \eta \ge a}\Paren{\bm \eta - a}$. Note that $\bm \eta \le a + \bm \xi$ and $\bm \xi$ is nonnegative, hence by \cref{fact:expectation_nonnegative_variable} we have
%	\[
%	\B\bm \xi = \int_0^\infty \Pr\Brac{\bm\xi > \tau} d\tau = \int_0^\infty \Pr\Brac{\bm\eta - a > \tau} d\tau = \int_0^\infty f(\tau) d\tau < \infty \,.
%	\] 
%	Hence either $\B \bm \eta = -\infty$ and the bound is trivially satisfied, or we can take the expectations form both sides of $\bm \eta \le a + \bm \xi$.
%\end{proof}

		\section{Linear Algebra}

\begin{lemma}\label{lem:linear-algebra-correlation-eigenverctor-large-quadratic-form}
	Let $M\in \R^{d\times d}$, $M\sge 0$, $\Tr M = 1$ and let $z\in \R^d$ be a unit vector such that $\transpose{z}Mz\geq 1-\eps$. Then the top eigenvector $v_1$ of $M$ satisfies $\iprod{v_1,z}^2\geq 1-2\eps$.
	\begin{proof}
		Write $z=\alpha v_1+\sqrt{1-\alpha^2}v_{\bot}$ where $v_\bot$ is a unit vector orthogonal to $v_1$. 
		\begin{align*}
		\transpose{z}Mz &= \alpha^2 \transpose{v_1}Mv_1+\Paren{1-\alpha^2}\transpose{v_\bot} M v_\bot\\
		&=\alpha^2 \Paren{\lambda_1-\transpose{v_\bot}Mv_\bot}+\transpose{v_\bot}M v_\bot \\
		&\geq 1-\eps
		\end{align*}
		As $\transpose{v_1}Mv_1\geq \transpose{z}Mz$ and 
		$\transpose{v_\bot}Mv_\bot\leq 1-\transpose{v_1}Mv_1 
		\le 1- \transpose{z}Mz\le\eps$, rearranging
		\begin{align*}
		\alpha^2 \geq \frac{1-\eps-\transpose{v_\bot}M v_\bot}{\lambda_1-\transpose{v_\bot}M v_\bot}\geq 1-2\eps.
		\end{align*}
	\end{proof}
\end{lemma}

\begin{fact}\label{fact:product-of-psd-matrices}
Let $A, B \in \R^{d\times d}$, $A,B\sge0$. Then $\iprod{A,B}\geq 0$.
\end{fact}

%\section{Missing proofs}\label{section:missing-proofs}
\begin{fact}\label{fact:basic-sdp-iprod-matrix-spectral-norm}
	Let $X\in \R^{d\times d}$ be a positive semidefinite matrix. 
	Then for all $A\in \R^{d\times d}$,
	\begin{align*}
	\Abs{\iprod{A,X}}\leq \norm{A}\cdot\Tr{X}\,.
	\end{align*}
\end{fact}

\begin{lemma}\label{lem:maximizer-lemma}
	Let $\Omega \subset \R^m$. Let $Y = S + N$, where $S\in \Omega$ and $N\in \R^m$ statisfies
	\[
	\sup_{X\in \Omega}\Abs{\Iprod{X, N}} \le \delta \,.
	\]
	Then $\hat{X} \in \argmax_{X \in \Omega} \Iprod{X, Y}$ satisfies
	\[
	\Iprod{\hat{X} , S} \ge \Norm{S}^2 - 2\delta\,.
	\]
\end{lemma}
\begin{proof}
	\[
	\Iprod{\hat{X} , S} = \Iprod{\hat{X} , Y}  - \Iprod{\hat{X} , N} 
	\ge \Iprod{\hat{X} , Y} - \delta \ge \Iprod{S, Y} -\delta = \Iprod{S, S} + \Iprod{S, N} - \delta \ge \Norm{S}^2 - 2\delta\,.
	\]
\end{proof}

\begin{fact}\label{fact:three-close-vectors}
	Let $a, b, c\in \R^m$ such that $\norm{a} = \norm{c} = 1$ and $\norm{b}\le 1$.
	Suppose that $\Iprod{a,b} \ge 1-\delta$ and $\Iprod{c,b} \ge 1 - \delta$. Then $\iprod{a,c}\ge 1 - 4\delta$.
\end{fact}
\begin{proof}
	Note that $\norm{a-b}^2 \le 2 - 2\iprod{a,b} \le 2\delta$ 
	and similarly $\norm{c-b}^2 \le  2\delta$. Hence
	\[
	2 - 2\Iprod{a,c} = \Norm{a-c}^2 \le \Paren{\Norm{a-b} + \Norm{c-b}}^2 \le 8\delta\,.
	\]
\end{proof}

\begin{lemma}\label{lem:perturbation-inner-product}
	Let $\cP \subset \R^{d\times d}$ be some set of PSD matrices.
	For matrix $M \in \R^{n\times d}$ let
	\[
	\mathfrak{s}_\cP(M) := \sup_{X\in \cP} \Abs{\Iprod{M^\top M, X}}\,.
	\]
	Then for arbitrary matrices $A, B\in \R^{n\times d}$, 
	\[
	\sup_{X\in \cP}\Abs{\Iprod{B^\top A + A^\top B, X}} \le 2\sqrt{\mathfrak{s}_\cP(A) \cdot \mathfrak{s}_\cP(B) }\,.
	\]
\end{lemma}
\begin{proof}
	Let $X'$ be an arbitrary element of $S$. 
	For some $c\in0$ (we will choose the value of $c$ later), let $C = \Paren{A-cB}^\top \Paren{A-cB}$.
	Since $C$ and $X'$ are PSD, $\Iprod{C, X'} \ge 0$. Hence
	\[
	\Abs{c\cdot \Iprod{B^\top A + A^\top B, X'}} \le \Iprod{A^\top A, X'} + c^2 \Iprod{B^\top B, X'} \le \mathfrak{s}_\cP(A) + c^2\mathfrak{s}_\cP(B)\,.
	\]
	Therefore,
	\[
	\Abs{\Iprod{B^\top A +A^\top B, X'}} \le \frac{\mathfrak{s}_\cP(A)}{\Abs{c}} + \Abs{c}\cdot \mathfrak{s}_\cP(B)\,.
	\]
	To minimize this expression, we can take $\abs{c} = \sqrt{\frac{\mathfrak{s}_\cP(A)}{\mathfrak{s}_\cP(B)}}$. Hence
	\[
	\Abs{\Iprod{B^\top A + A^\top B, X'}} \le  2\sqrt{\mathfrak{s}_\cP(A)\cdot \mathfrak{s}_\cP(B)} \,.
	\]
	Since it holds for arbitrary $X'\in \cS$, we get the desired bound.
\end{proof}

\section{Concentration Inequalities}

\begin{fact}[Chernoff's inequality, \cite{vershynin_2018}]\label{fact:chernoff}
	Let $\bm \zeta_1,\ldots, \bm \zeta_n$ 
	be independent Bernoulli random variables such that 
	$\Pr\Paren{\bm \zeta_i = 1} = \Pr\Paren{\bm\zeta_i = 0} = p$. 
	Then for every $\Delta > 0$,
	\[
	\Pr\Paren{\sum_{i=1}^n \bm\zeta_i \ge pn\Paren{1+ \Delta} } 
	\le 
	\Paren{ \frac{e^{-\Delta} }{ \Paren{1+\Delta}^{1+\Delta} } }^{pn}\,.
	\]
	and for every $\Delta \in (0,1)$,
	\[
	\Pr\Paren{\sum_{i=1}^n \bm\zeta_i \le pn\Paren{1- \Delta} } 
	\le 
	\Paren{ \frac{e^{-\Delta} }{ \Paren{1-\Delta}^{1-\Delta} } }^{pn}\,.
	\]
\end{fact}

\begin{fact}[Hoeffding's inequality, \cite{wainwright_2019}]\label{fact:hoeffding}
	Let $\bm z_1,\ldots, \bm z_n$
	be mutually independent random variables such that for each $i\in[n]$,
	$\bm z_i$ is supported on $\brac{-c_i, c_i}$ for some $c_i \ge 0$. 
	Then for all $t\ge 0$,
	\[
	\Pr\Paren{\Abs{\sum_{i=1}^n \Paren{\bm z_i - \E \bm z_i}} \ge t} 
	\le 2\exp\Paren{-\frac{t^2}{2\sum_{i=1}^n c_i^2}}\,.
	\]
\end{fact}
\begin{fact}[Bernstein's inequality
	\cite{wainwright_2019}]\label{fact:bernstein}
	Let $\bm z_1,\ldots, \bm z_n$
	be mutually independent random variables such that for each $i\in[n]$,
	$\bm z_i$ is supported on $\brac{-B, B}$ for some $B\ge 0$. 
	Then for all $t\ge 0$,
	\[
	\Pr\Paren{{\sum_{i=1}^n \Paren{\bm z_i - \E \bm z_i}} \ge t} 
	\le \exp\Paren{-\frac{t^2}{2\sum_{i=1}^n \E \bm z_i^2 + \frac{2Bt}{3}}}\,.
	\]
\end{fact}

\begin{fact}\cite{wainwright_2019}\label{fact:gaussian-tail-bound}
	Let $X\sim N(0,\sigma^2)$,  then for all $t>0$,
	\begin{align*}
	\bbP \Paren{X \geq \sigma\cdot t} &\leq e^{-t^2/2}\,.
	\end{align*}
\end{fact}

\begin{fact}\cite{laurent2000}\label{fact:chi-squared-tail-bounds}Let $X\sim \chi^2_m$,  then for all $x>0$,
	\begin{align*}
	\bbP \Paren{X-m \geq 2x+2\sqrt{mx}} &\leq e^{-x}\\
	\bbP \Paren{m-X\geq x}&\leq e^{-\frac{x^2}{4m}}
	\end{align*}
\end{fact}

\begin{fact}\cite{wainwright_2019}\label{fact:bound-covariance-gaussian}
	Let $W\sim N(0,1)^{n \times d}$. Then with probability $1-\exp\Paren{-t/2}$,
	\[
	\Norm{W}\le \sqrt{n}  + \sqrt{d} + \sqrt{t}\,
	\]
	and
	\[
	\Norm{\transpose{W}W-n \Id} \le d + 2\sqrt{dn} + t + 4\sqrt{t(n+d)} \,.
	\]
\end{fact}

\begin{fact}\cite{dKNS20}\label{fact:k-sparse-norm-gaussian}
	Let $\bm W\sim N(0,1)^{n\times d}$ be a Gaussian matrix. Let $1\le k \le d$. Then with probability at least $1-\Paren{\frac{k}{ed}}^k$
	\[
	  \max_{\substack{u\in\R^n\\ \norm{u}=1}}\;\;
	  \max_{\substack{\text{$k$-sparse }v\in\R^d\\ \norm{v}=1}} \transpose{u}\bm Wv \le
	  \sqrt{n} + 3\sqrt{k\ln\Paren{\frac{ed}{k}}}
	\]
	and
	\[
	\max_{\substack{\text{$k$-sparse }v\in\R^d\\ \norm{v}=1}}  v^\top \bm W^\top \bm W v - n \le 10\sqrt{kn\ln\Paren{ed/k}} + 10k\ln\Paren{ed/k}\,.
	\]
\end{fact}

\begin{lemma}\label{lem:spectral-bound}
	Let $\bm Y$ be an instance of sparse PCA in Wishart model.
	For $m\in \N$ let $Z_m = \Set{z \in \Set{0,1}^d \suchthat \normo{z} \le m}$.
	Suppose that $m \ge 100\ln d$ and $n \ge 0.1\cdot  m\ln\Paren{ed/m}$.
	Then, with probability at least $1-d^{-10}$,
	\[
	\max_{z\in Z_m}\, \Norm{\Paren{\bm Y^\top \bm Y - n\Id -\beta\norm{\bm u}^2vv^\top} \circ \Paren{zz^\top}} 
	\le 10 \sqrt{\Paren{n+\beta n}\cdot m\ln\Paren{ed/m}} + 10m\ln\Paren{ed/m}\,.
	\]
\end{lemma}
\begin{proof}
	We can write $\Paren{\bm Y^\top \bm Y - n\Id -\beta\norm{\bm u}^2vv^\top} \circ \Paren{zz^\top}$ as
	\[
	\Paren{\bm W^\top \bm W - n\cdot\Id  + \sqrt{\beta} \bm W^\top \bm u v^\top +  \sqrt{\beta} v\bm u^\top \bm W}
	\circ \Paren{z z^\top}\,.
	\] 
	
	Note that
	\[
	\Norm{\Paren{\bm W^\top \bm u v^\top +  v\bm u^\top \bm W } \circ  \Paren{z z^\top}}\le
	2 \Norm {\Paren{\frac{1}{\Norm{\bm u}} \bm u^\top \bm W}\circ z}\cdot\Norm{\bm u}\,.
	\]
	With probability at least $1-\exp\paren{n/10}$, $\Norm{\bm u} \le \sqrt{2n}$.
	By \cref{lem:spectral-norm-cross}, with probability at leaast $1-\exp\paren{-m}$,
	\[
	\Norm{\Paren{\bm W^\top \bm u v^\top +  v\bm u^\top \bm W } \circ  \Paren{z z^\top}}\le
	5\sqrt{\beta n} \cdot \sqrt{m \ln\Paren{em/k}}\,.
	\]

	By \cref{fact:bound-covariance-gaussian}, for every $m'\in \N$,
	$\bm G \sim N(0,1)^{n\times m'}$ satisfies
	\[
	\Norm{\bm G^\top \bm G - n\cdot\Id} \le m' + 2\sqrt{m'n} + \tau + 4\sqrt{\tau\Paren{m' + n}}\,.
	\]
	with probability $1-\exp(-\tau/2)$ (for every $\tau > 0$).
	By union bound over all sets of size $m' \le m$ 
	we get the desired bound.
\end{proof}

\begin{lemma}\label{lem:spectral-norm-cross}
	Let $\bm W \sim N(0,1)^{n\times d}$ 
	and let $a\in \R^n$ and $b\in \R^d$ be vectors independent of $\bm W$. 	
	For $m\in \N$ let $Z_m = \Set{z \in \Set{0,1}^d \suchthat \normo{z} \le m}$.
	Suppose that $m\ge 100\ln d$.
	Then with probability at least $1-d^{-10}$,
	\[
	\max_{z\in Z_m}\, \Norm{\Paren{b a^\top \bm W}\circ \Paren{zz^\top}}
	\le 3\cdot  \norm{a}\cdot\norm{b}\cdot \sqrt{m \ln\Paren{em/k}}\,.
	\]
\end{lemma}
\begin{proof}
	For fixed set $J \subset [d]$, let $\textbf{1}_J \in \Set{0,1}^d$ be an indicator vector of this set. Random vector
	$\Paren{\frac{1}{\Norm{\bm u}} \bm u^\top \bm W}\circ \textbf{1}_J$ has standard $\Card{J}$-dimensional Gaussian distribution,
	hence by \cref{fact:bound-covariance-gaussian},
	 for all $\tau > 0$, with probability at least $1-\exp\paren{-\tau/2}$,
	\[
	\Norm{\Paren{\frac{1}{\Norm{a}} a^\top \bm W}\circ \textbf{1}_J} \le \sqrt{\card{J}} + 1 + \sqrt{\tau}\,.
	\]
	By union bound over all sets of size of size at most $m$, we get
	\[
	\max_{z\in Z_m}\,\Norm {\Paren{\frac{1}{\Norm{a}} a^\top \bm W}\circ z} \le 3\sqrt{m\cdot \ln\Paren{ed/m}}
	\]
	with probability at least $1-\exp\paren{-m}$.
\end{proof}

\end{document}